\documentclass[a4paper,USenglish,cleveref,showAuthorEmail,showAuthorOrcid,hideChapterNumber,refchapterCite]{lipics-v2021}
\usepackage{craigstyle}
\usepackage[T1]{fontenc}
\usepackage{xspace}
\usepackage{misc}
\usepackage{amsmath}
\usepackage{amssymb}
\usepackage{enumitem}
\usepackage{csquotes}
\usepackage{graphicx}
\usepackage{todonotes}
\usepackage{bbm}
\usepackage{stmaryrd}
\usepackage[shortcuts]{extdash} 
\usepackage{hyperref}
\usepackage{nameref}
\usepackage{color}

\makeatletter
\let\orgdescriptionlabel\descriptionlabel
\renewcommand*{\descriptionlabel}[1]{%
  \let\orglabel\label
  \let\label\@gobble
  \phantomsection
  \edef\@currentlabel{#1}%
  \let\label\orglabel
  \orgdescriptionlabel{#1}%
}
\makeatother

\newcommand{\patrick}[1]{\todo[inline,color=orange]{Patrick: #1}}

\newcommand{\added}[1]{#1}

\newcommand{\sig}[1]{\textsf{sig}(#1)}
\newcommand{\tp}{\textsf{tp}}
\newcommand{\size}[1]{\lVert #1 \rVert}
\newcommand{\rd}{\text{rd}}

\newcommand{\nb}[1]{\textcolor{red}{!}\marginpar{\textcolor{red}{#1}}}

\newcommand{\hide}[1]{}

\newcommand{\tup}[1]{\langle#1\rangle}

\newcommand{\nc}{\,\mid\!\sim\,}
\newcommand{\nccw}{\,\mid\!\sim_{c}}

\newcommand{\sign}{{\Sigma}}

\newcommand{\setm}{\text{\textbackslash}}

\newcommand{\htmodels}{\models_{\sf HT}}

\newcommand{\HT}{\mathcal{HT}}

\newcommand{\classP}{\mathcal{C}}
\newcommand{\la}{\leftarrow}
\newcommand{\nf}{not\,}

\newcommand{\head}[1]{\ensuremath{\mathit{H}}}
\newcommand{\pbody}[1]{\ensuremath{\mathit{B^+}}}
\newcommand{\nbody}[1]{\ensuremath{\mathit{B^-}}}
\newcommand{\nnbody}[1]{\ensuremath{\mathit{B^{--}}}}

\newcommand{\rhead}[1]{\ensuremath{\mathit{head}(#1)}}
\newcommand{\rbody}[1]{\ensuremath{\mathit{body}(#1)}}

\newcommand{\vsim}{\mathrel{\scalebox{1}[1.5]{$\shortmid$}\mkern-3.1mu\raisebox{0.15ex}{$\sim$}}}

\newcommand{\mtuple}[1]{\ensuremath{\langle#1\rangle}}

\newcommand{\pW}{{\bf (W)}}
\newcommand{\pPP}{{\bf (PP)}}

\newcommand{\pCP}{{\bf (CP)}}
\newcommand{\pSP}{{\bf (SP)}}

\newcommand{\f}[2]{\ensuremath{\mathsf{f}(#1,#2)}}
\newcommand{\fgt}{\ensuremath{\mathsf{f}}}
\newcommand{\classF}{\ensuremath{\mathsf{F}}}

\newcommand{\as}[1]{\ensuremath{\mathcal{AS}(#1)}}

\newcommand{\FS}{S}

\newcommand{\htF}{{\sf HT}}
\newcommand{\smF}{{\sf SM}}

\newcommand{\rF}{\sf R}

\newcommand{\mF}{\sf M}

\newcommand{\up}{\sf UP}

\newcommand{\op}{{\mathsf{f}}}

\newcommand{\PiP}[1]{{{\Pi}_{#1}^{P}}}

\newcommand{\SO}{\text{SOL}\xspace}

\newcommand{\Lethe}{\textsc{Lethe}\xspace}
\newcommand{\Fame}{\textsc{Fame}\xspace}

\title{Interpolation in Knowledge Representation}
\author{Jean Christoph Jung}{TU Dortmund University, Germany}{jean.jung@tu-dortmund.de}{https://orcid.org/0000-0002-4159-2255}{}{}
\author{Patrick Koopmann}{Vrije Universiteit Amsterdam, Netherlands}{p.k.koopmann@vu.nl}{https://orcid.org/0000-0001-5999-2583}{}{}
\author{Matthias Knorr}{Universidade Nova de Lisboa, Portugal}{mkn@fct.unl.pt}{https://orcid.org/0000-0003-1826-1498}{}{}
\authorrunning{J.C. Jung, P. Koopmann, M. Knorr}
\begin{document}
\maketitle              

\begin{abstract}
 Craig interpolation and uniform interpolation have many applications in knowledge representation,
 including explainability, forgetting, modularization and reuse, and even learning.
 At the same time, many relevant knowledge representation formalisms do in general not have Craig or uniform
 interpolation, and computing interpolants in practice is challenging. We have a closer look at two prominent
 knowledge representation formalisms, description logics and logic programming, and discuss
 theoretical results and practical methods for computing interpolants.
\end{abstract}

\tableofcontents

\newcommand{\blue}[1]{\textcolor{blue}{#1}}

\section{Introduction}\label{sec:intro}

The field of Knowledge Representation and Reasoning (KR) deals with
the explicit representation and manipulation of knowledge in a format that is both \emph{machine-processable} and \emph{human-readable}, and plays a central role in symbolic AI~\cite{HandbookKR,DBLP:books/daglib/0023546}.
%
%
AI systems using KR 
are referred to as \emph{knowledge-based systems}. Depending on the application, they use different \emph{knowledge representation formalisms} satisfying the mentioned two conditions.
Since interpolation is an inherently logical notion, in this chapter, we will
concentrate on logic-based formalisms, and neglect, e.g., graph-based KR formalisms such as argumentation frameworks~\cite{ARGUMENTATION_FRAMEWORKS} or probabilistic graphical models~\cite{PGMs}. Knowledge-based systems then implement inference procedures tailored to the underlying logic and use them to make decisions.

Let us illustrate the idea underlying KR with a simple example, formulated in terms of first-order logic. The following statements could be part of a formalization of our knowledge about cars that (1) every car has a prime mover, (2) the prime mover can be a diesel engine, a gas engine or an electric motor, and (3) every car with an electric motor is an electric car:
\begin{align}
  \forall x\,& \left(\text{Car}(x)\to \exists y\, (\text{hasPart}(x,y)\wedge \text{PrimeMover}(y))\right)\\
  \forall x\,& \left(\text{PrimeMover}(x)\to (\text{DieselEngine}(x)\vee \text{GasEngine}(x)\vee \text{ElectricMotor}(x)) \right)\\
\forall x\,& \left(\text{Car}(x)\land\exists y\, (\text{hasPart}(x,y)\wedge \text{ElectricMotor}(y))\to \text{ElectricCar}(x)\right)
\end{align}
If we now find out that a car $c$ has an electric motor, we can feed this knowledge and additional facts
\[\text{Car}(c),\quad\text{hasPart}(c,e),\quad\text{ElectricMotor}(e)\]
into the inference procedure, which allows us to derive that $c$ is an electrical car.

Interpolation plays a central role in many applications of KR. To ease the discussion, we recall the definition of Craig and uniform interpolants, formulated here for first-order logic (FO) with entailment relation $\models$. Here, and in the remainder of the introduction, a \emph{signature} is a set of non-logical symbols.


\begin{definition}\label{def:general-interpolants}
    A \emph{Craig interpolant for}
FO-formulae $\phi,\psi\in$ is a formula $\chi\in\text{FO}$ such that
 \begin{itemize}
  \item $\chi$ uses only non-logical symbols occurring in both $\phi$ and $\psi$, and
  \item $\phi\models\chi$, $\chi\models\psi$.
 \end{itemize}
 Let $\Sigma$ be signature.
 A \emph{uniform $\Sigma$-interpolant for $\phi$} is an FO-formula $\phi_\Sigma$ such that: 
 \begin{itemize}
  \item $\phi\models\phi_\Sigma$;
  \item $\phi_\Sigma$ uses only non-logical symbols from $\Sigma$;
  \item for any FO-formula $\psi$, 
  if $\phi\models\psi$ and $\psi$ uses only non-logical symbols from $\Sigma$, then $\phi_\Sigma\models\psi$.
 \end{itemize}
\end{definition}



In applications, the formula $\phi$ from Definition~\ref{def:general-interpolants} will typically be the modeled knowledge, which is why we will denote it with \Kmc, for \emph{knowledge base}. Alghough knowledge bases come in different forms for different KR formalisms, we refrain from specifying the exact formalism in the applications below, as we only want to convey the general ideas. We start with applications of uniform interpolants.

\smallskip
\textbf{Forgetting and Information Hiding.} The most classical connection between interpolation and KR might be \emph{forgetting}, originally introduced for first-order logic in a seminal paper by Lin and Reiter~\cite{Reiter80}. The idea of forgetting is that there are non-logical symbols that are considered ``obsolete'' and should be ``forgotten'' from the current knowledge in a way that preserves as much information as possible. There are different definitions of forgetting that can be found in the KR literature, see~\cite{FORGETTING_SURVEY} for a survey, but the most common one is equivalent to uniform interpolation: forgetting a symbol $X$ from a knowledge base $\Kmc$ corresponds to computing a uniform interpolant of \Kmc for the signature $\sig{\Kmc}\setminus\{X\}$. 
Forgetting is also useful for applications that require some form of \emph{information hiding}, such as knowledge publication and exchange.
For instance, if we were to publish $\Kmc$ without disclosing anything about a private signature $\Sigma$, then a uniform $(\sig{\Kmc}\setminus\Sigma)$-interpolant of \Kmc would be a good candidate, as is contains the \enquote{maximal} information contained in \Kmc about the remaining signature~\cite{DBLP:journals/semweb/Grau10}.


\smallskip
\textbf{Abduction.} Abduction is a classical problem in KR that can be formalized as follows: given a knowledge base $\Kmc$ and an \emph{observation} $\psi$ such that $\Kmc\not\models\psi$, we want to find a \emph{hypothesis} $\Hmc$ satisfying $\Kmc\wedge\Hmc\models\psi$. The hypothesis $\Hmc$ can be seen as a possible explanation for some unexpected phenomenon $\phi$, a diagnosis, or an indication of how to complete an incomplete knowledge base. To avoid trivial solutions, one typically formulates additional requirements on the hypothesis such as certain minimality conditions or a signature $\Sigma$ from which the hypothesis ought to be constructed. It is immediate from the definitions that the negation of a $\Sigma$-uniform interpolant of $\Kmc\wedge\neg \psi$ is a \emph{logically weakest} hypothesis, that is, it is entailed by all other alternative hypotheses over signature~$\Sigma$. This connection has been exploited in several abduction systems~\cite{SOQE_BOOK,ABDUCTION_PINTO,MY_ABDUCTION}.

\smallskip \textbf{Modularisation and Reuse.}
Knowledge bases often contain tens or even hundreds of thousands of statements
(e.g.~\cite{donnelly2006snomed,konopka2015biomedical})
which can make it challenging to work with them. At the same time, it could be that for a particular application, only a fragment of the knowledge base is actually relevant. Uniform interpolants computed for a restricted, user-given signature can provide a more focussed view on the knowledge that is relevant to the application at hand, and can be used as a replacement of the original knowledge base. But uniform interpolants can also be used to determine whether a subset of the knowledge base preserves all entailments over a relevant signature, which can be used to \emph{modularize} the knowledge base into different components based on a signature.
%

\smallskip\textbf{Analysis and Summarization.}
Uniform interpolants for small signatures $\Sigma$ can make hidden relations between the symbols in $\Sigma$ explicit, and thus help knowledge engineers in analyzing how different symbols in the knowledge base relate to each other. We can also use uniform interpolation to summarize large knowledge bases, for instance by choosing for~$\Sigma$ a set of symbols that is central to the knowledge base (e.g. most frequently used). The interpolant then gives a high-level perspective on the content of the knowledge base, which can be seen as a summary. 


\bigskip
Craig interpolants have a range of applications for KR as well.

\smallskip
\textbf{Explaining Entailments.}
One of the main benefits of KR is that the use of explicit representations of knowledge with a well-defined semantics leads to a higher transparency and understandability of the knowledge-based system. Still, for large knowledge bases and expressive KR formalisms, inferences may not be straightforward, which has motivated many researchers to investigate methods for explaining logical inferences for KR systems~\cite{DBLP:conf/ijcai/McGuinnessB95,DBLP:conf/jelia/Schlobach04,DBLP:journals/tplp/FandinnoS19,DBLP:conf/lpnmr/TrieuSB22,DBLP:conf/kr/AlrabbaaBFHKKKK24}.
Craig interpolants are one way of explaining logical inferences: in particular, to explain an entailment $\Kmc\models\psi$, we can provide a Craig interpolant for $\Kmc,\psi$, which highlights the reason for the entailment in the common signature of $\Kmc$ and $\psi$. Notice that the existence of a Craig interpolant for $\Kmc,\psi$ is a weaker requirement than the existence of a uniform interpolant for~\Kmc, and indeed Craig interpolants can be computed in more cases. 

\smallskip
\textbf{Explicit Definitions.} It is well-known that Craig interpolation is closely related to the notion of \emph{Beth definability}. In the context of KR, this connection is particularly relevant for the problem of finding \emph{explicit definitions} of predicates which are \emph{implicitly defined}. An explicit definition of a unary predicate $A$ in knowledge base \Kmc is a formula $\chi(x)$ not mentioning $A$ with $\Kmc\models \forall x\,(A(x)\leftrightarrow \chi(x))$. That is, $\chi$ provides a direct description of the \emph{meaning} of $A$, and we might add this definition to \Kmc. A special kind of explicit definitions are \emph{referring expressions}. Referring expressions are a concept originating in linguistics and are phrases that refer to specific objects by providing a unique description for them, as in ``the current president of the USA'' or ``the capital of France''. In a knowledge base \Kmc, a referring expression for a constant $c$ can be viewed as an explicit definition of $x=c$ over \Kmc.
The use of referring expressions in KR and data management has been advocated for instance in~\cite{DBLP:conf/inlg/ArecesKS08,DBLP:journals/coling/KrahmerD12,DBLP:conf/kr/BorgidaTW16,ArtEtAl21}.

\smallskip \textbf{Separating Examples and Learning Logical Formulas.} Consider the 
following separation problem. Suppose we are given two sets $P$ and
$N$ of constants occurring in a knowledge base $\Kmc$ that respectively represent \emph{positive} and \emph{negative} examples. We are then looking for a
logical formula $\chi(x)$ that \emph{separates} $P$ from $N$ over \Kmc in the sense that $\Kmc\models\chi(a)$ for all $a\in P$ and
$\Kmc\models\neg\chi(b)$ for all $b\in N$. Such separation problems have
been investigated thoroughly in
KR in the context of learning logical formulas~\cite{DBLP:journals/ai/JungLPW22,DBLP:conf/ijcai/FunkJLPW19}. Recently, it
has been observed that in relevant cases, the problem of finding a separating
formula is interreducible with the problem of computing Craig
interpolants~\cite{DBLP:journals/tocl/ArtaleJMOW23}. 


\bigskip

Motivated by these applications, there has been an enormous interest in studying interpolation in different formalisms.  
The baseline logical formalisms used in KR are certainly the classical \emph{propositional logic} and \emph{first-order logic}, which have the benefit of a clean and well-understood semantics. However,
for many realistic applications, neither of them is well-suited, since they are either limited in their expressivity, or
do not admit time-efficient reasoning.
As a consequence, a wealth of other formalisms have been introduced that are tailored towards specific applications. Examples include description logics, modal and temporal logics, logic programs, default logic, existential rules, planning languages, and many more~\cite{HandbookKR}.
Interpolation is useful in many of these formalisms, but it has not been equally investigated in all of them.
Interpolation in propositional logic is covered in \refchapter{chapter:propositional}, interpolation in modal logic is covered in \refchapter{chapter:sixproofs}, and interpolation in first-order logic is covered in \refchapter{chapter:firstorder} and \refchapter{chapter:automated}. The majority of research on interpolation in the remaining formalisms has been conducted in description logics and logic programming, which is why we concentrate on these two formalisms in this chapter:
\begin{itemize}
 \item \emph{Description Logics (DLs)} are a family of KR languages commonly used to formalize ontologies~\cite{DL-Textbook}. Here, an \emph{ontology} is a formal specification of the concepts and their relations in some domain of interest such as biology or medicine. An example ontology (in first-order logic) is provided in Equations~(1)--(3) above.
 Ontologies are important in information science, since they can be used by different parties to share knowledge about that domain, which has been exploited, for example, in biology, medicine, and artificial intelligence~\cite{konopka2015biomedical,donnelly2006snomed,DBLP:phd/hal/Gandon02,DBLP:conf/aaai/Ribeiro021}.
 DLs are also highly relevant in the \emph{Semantic Web}~\cite{DBLP:journals/cacm/Horrocks08}. Indeed, they form the logical basis of the W3C standard \emph{web ontology language} OWL~\cite{DBLP:journals/ws/HorrocksPH03}. Typically, DLs are fragments of first-order logic with decidable inference problems, which makes them suitable for the mentioned applications.

 \item \emph{Logic programming} is concerned with the use of \emph{logical rules} to represent knowledge. One of the main features is that inference is \emph{nonmonotonic}, that is, extending the  knowledge base can invalidate previously made inferences. This is in contrast to \emph{monotonic} logics such as first-order logic (and hence DLs as well). 
 While logic programming has been studied since the 1960s, it is still relevant today, in particular in the form of \emph{answer set programming} (ASP), which is a declarative approach to modeling and solving combinatorial problems~\cite{ASPPractice12}. It plays a central role in many applications, for instance for solving configuration problems.


%
\end{itemize}
It is particularly remarkable how much has been done on the topic of interpolation in DLs. We conjecture that this is due to the fact that the main applications of DLs are ontologies describing the conceptualization of a domain of discourse, a task that is ultimately linked with the signature, which is also central to interpolation. In the chapter, this will be reflected by the fact that we will mainly focus on interpolation for DLs, and only briefly discuss interpolation for logic programming. A peculiarity in the literature on interpolation in DLs, which has to do with the applications, is that uniform interpolation has been investigated mostly for the case where the knowledge base is an ontology, and Craig interpolation for DLs has been mostly investigated for concept descriptions, in which case the ontology is treated as background theory.

Looking at the applications, it appears that the central problem to be investigated is the \emph{computation} of Craig and uniform interpolants. Unfortunately, their existence is by no means guaranteed in the relevant logics, so another focus of the chapter will be the corresponding \emph{existence} problems, and the practically relevant question of what to do if no (Craig or uniform) interpolant exists. Given the technical similarity of description logics and modal logics, we will discuss the concrete relation (in terms of interpolation) when appropriate.

The chapter is structured as follows. In Section~\ref{sec:description-logics}, we recall the necessary foundations of description logics. In Section~\ref{sec:ui-global}, we
discuss uniform interpolation for the case where the knowledge base is a description logic ontology, looking at both theoretical results and practical methods for computing interpolants.
In Section~\ref{sec:interpolation-concept-level}, we discuss results on interpolation and Beth definability for DL concept descriptions from a more theoretical perspective.
In Section~\ref{sec:asp}, we discuss the role of Craig interpolants in logic programming, as well as uniform interpolation and forgetting.
Finally, in Section~\ref{sec:conclusion}, we conclude the chapter and provide an outlook for future directions. 

\bigskip

\hide{
\paragraph{Overview of the chapter.}

As we announced before different forms of interpolation, but also
explicit definitions are important. This induces the sections of the
chapter. In the next section, we will introduce the relevant DLs. Then
we will discuss classical Craig interpolation and Beth definability.
After that, we will give an overview of uniform interpolation.
Finally, as announced above, we will survey results on interpolation
in other subfields of KR. 

\begin{itemize}
    \item Discuss applications including other logical formalisms (classical logics and modal logics)
    \begin{itemize}
        \item Example: forgetting, communication in multi-agent systems
    \end{itemize}
    \item CI and UI for classical and modal logics discussed in other chapters (make references)
    \item nowadays, most relevant KR formalisms are DLs and ASP - hence we focus on those
\end{itemize}
}

\section{Foundations of Description Logics}\label{sec:description-logics}

We first introduce the syntax and semantics of the basic description
logic~\ALC, discuss the extensions/restrictions relevant for the chapter, and introduce some model theory. We refer the reader to~\cite{DL-Textbook} for a comprehensive introduction to description logics.
Let \NC, \NR, and~\NI be mutually disjoint and countably infinite sets
of \emph{concept}, \emph{role}, and \emph{individual names}.
An \emph{\ALC concept} is defined according to the
syntax rule
\[
C, D ::= \top \mid A \mid \neg C \mid C \sqcap D \mid C\sqcup D\mid \forall r.C \mid \exists r.C
\]
where $A$ ranges over concept names and $r$ over role names.
We use the abbreviations $C\to D$ and $C\leftrightarrow D$ for
$\neg C\sqcup D$, and
$(C\to D)\sqcap (D\to C)$, respectively.
An \emph{\ALC concept inclusion (\ALC CI)}
takes the form $C \sqsubseteq D$ for \ALC concepts $C$ and $D$.
An \emph{$\ALC$ ontology} is a finite set of \ALC CIs. We drop the
reference to \ALC if no confusion can arise.

The semantics is defined in terms of \emph{interpretations}
$\Imc=(\Delta^\Imc,\cdot^\Imc)$,
where $\Delta^{\mathcal{I}}$ is a non-empty set, called \emph{domain} of $\mathcal{I}$,
and $\cdot^{\mathcal{I}}$ is a function mapping
every $A \in \NC$ to a subset $A^\Imc\subseteq \Delta^{\Imc}$,
every $r\in\NR$ to a subset $r^\Imc\subseteq \Delta^{\Imc} \times
\Delta^{\Imc}$,
and every $a \in\NI$
to an element in $\Delta^{\Imc}$.
Moreover, the \emph{extension $C^{\mathcal{I}}$ of a concept
$C$ in $\mathcal{I}$} is defined as follows, where $r$ ranges over
role names:
\begin{align*}
	\top^{\mathcal{I}} & = \Delta^{\mathcal{I}}, \\
	\neg C^{\mathcal{I}} & = \Delta^{\Imc} \setminus C^{\mathcal{I}}, \\
	(C \sqcap D)^{\mathcal{I}} & = C^{\mathcal{I}} \cap D^{\mathcal{I}}, \\
	(C \sqcup D)^{\mathcal{I}} & = C^{\mathcal{I}} \cup D^{\mathcal{I}}, \\
	(\exists r.C)^{\mathcal{I}} & = \{d \in \Delta^{\Imc} \mid \text{there exists $(d,e) \in r^\Imc$ with $e \in C^\Imc$ }\}, \\
	(\forall r.C)^{\mathcal{I}} & = \{d \in \Delta^{\Imc} \mid \text{for all $(d,e) \in r^\Imc$ we have $e \in C^\Imc$ }\}.
\end{align*}
Let \Omc be an ontology and \Imc be an interpretation. 
Then \Imc \emph{satisfies} a CI $C \sqsubseteq D$ if $C^\Imc \subseteq
D^\Imc$, and \Imc is a \emph{model} of \Omc if it satisfies all CIs
in~\Omc. The ontology $\Omc$ \emph{entails} a CI $C\sqsubseteq D$,
in symbols $\Omc\models C\sqsubseteq D$, if every
model of $\Omc$ satisfies $C\sqsubseteq D$; it \emph{entails} another ontology $\Omc'$, written $\Omc\models\Omc'$ if $\Omc\models C\sqsubseteq D$ for every CI $C\sqsubseteq D\in \Omc'$. In case $\Omc\models C\sqsubseteq D$, we also say that \emph{$C$ is subsumed by $D$ under \Omc}. 
We drop $\Omc$ if it is empty and write $\models C \sqsubseteq D$ for $\emptyset\models C \sqsubseteq D$.


A \emph{signature} $\Sigma$ is a finite set of concept, role, and individual
names, uniformly referred to as \emph{symbols}.
We use $\text{sig}(X)$ to denote the set of symbols used
in any syntactic object $X$ such as a concept or an ontology.  An
\emph{$\ALC(\Sigma)$ concept} is an $\ALC$ concept $C$ with
$\text{sig}(C)\subseteq \Sigma$.
The \emph{size} of a
(finite) syntactic object $X$, denoted $\size{X}$, is the number of
symbols needed to represent it as a word.
The \emph{role depth} $\rd(C)$ of a concept $C$ is the maximal nesting depth of existential and universal restrictions in $C$.

\begin{table}[t]
  \centering
  \begin{tabular}{c|c|c}
    acronym & syntax & semantics \\[1mm]
    \hline
    \Imc & $\exists r^-.C$ & $\{ d\in \Delta^\Imc \mid \text{there
    exists $e\in C^\Imc:(e,d)\in r^\Imc$}\}$ \\[1mm]
    \Omc & $\{a\}$ & $\{a^\Imc\}$ \\[1mm]
    %
	\Fmc & $(\leq 1\ r)$ & $\{d\in \Delta^\Imc \mid \text{there is at most one $e\in \Delta^\Imc$ with $(d,e)\in r^\Imc$}\}$ \\[1mm]
	\Smc & $\textsf{trans}(r)$ & $r^\Imc$ is transitive \\[1mm]
    \hline\hline
    \Hmc & $r\sqsubseteq s$ & $r^\Imc\subseteq
    s^\Imc$ \\[1mm]
  \end{tabular}
  \caption{Extensions of \ALC}
  \label{tab:extensions}
\end{table}

\textbf{Extensions and Restrictions of \ALC.} Depending on the application, one
requires more or less expressive power to describe the domain knowledge,
which is why different extensions and restrictions of \ALC have been
investigated. The most prominent restriction is certainly \EL, in which only the
concept constructors $\top$, $A$, $C\sqcap D$, $\exists r.C$ are allowed (and thus no
negation $\neg C$, disjunction $C\sqcup D$, and universal restriction $\forall
r.C$)~\cite{DBLP:conf/ijcai/BaaderBL05}. We also consider several important
extensions, see~\cite{DL-Textbook} for more details on these. 
\emph{Inverse roles} (abbreviated with the acronym $\Imc$) take the form $r^-$
for a role name $r$. They are interpreted as the inverse of the interpretation
of the role, that is, $(r^-)^\Imc = \{(e,d)\mid (d,e)\in r^\Imc\}$ and can be
used to travel edges in the converse direction.
\emph{Nominals} (acronym \Omc) take the form $\{a\}$ for individual names
$a\in\NI$ and are used to describe singleton concepts.
\emph{(Local) functionality restrictions} (acronym \Fmc)
take the form $(\leq 1\ r)$ and describe the set of all individuals that have at most one $r$-successor.
\emph{Role hierarchies} (acronym \Hmc)
take the form $r\sqsubseteq s$ and can be used to describe relations between
roles.
Finally, some of the roles may be declared as \emph{transitive } (acronym \Smc),
which means that they have to be interpreted as transitive relations.
Table~\ref{tab:extensions} displays an overview of
the extensions and their semantics. The names of the extensions are obtained in a canonical way by appending the acronym of the additional constructor, e.g., \ALCHI is the extension of \ALC
with role hierarchies and inverse roles. A bit of care has to be taken when combining certain features, e.g., transitive roles with functionality restrictions; we adopt the standard restrictions known from the literature~\cite{DL-Textbook}.
%
\newcommand{\dl}[1]{\textsf{#1}}

\begin{example}
 The example from the introduction can be represented as follows as an \ALC ontology. The first and last axioms
 are also in \EL.
\begin{align*}
 &\dl{Car}\sqsubseteq\exists\dl{hasPart}.\dl{PrimeMover}\qquad
 \dl{PrimeMover}\sqsubseteq\dl{DieselEngine}\sqcup\dl{GasEngine}\sqcup\dl{ElectricMotor} \\
 &\dl{Car}\sqcap\exists\dl{hasPart}.\dl{ElectricMotor}\sqsubseteq\dl{ElectricCar}
\end{align*}
In $\mathcal{ALCFIO}$, we can additionally express that every engine belongs to at most one car, and that a German car is a
car built in Germany.
\begin{align*}
 \dl{PrimeMover}\sqsubseteq(\leq 1\ \dl{hasPart}^-) \qquad
 \dl{GermanCar}\equiv\dl{Car}\sqcap\exists\dl{madeIn}.\{\dl{Germany}\}
\end{align*}

\vspace{-.8cm}
\lipicsEnd 
\end{example}

\smallskip \textbf{Model Theory.}
We next recall the model-theoretic notion of a bisimulation, which is very useful in the context of interpolation in DLs.  
Let $\Sigma$ be a signature and $\Imc_1,\Imc_2$ be interpretations. A relation $Z \subseteq \Delta^{\Imc}\times \Delta^{\Jmc}$ is an
\emph{$\mathcal{ALC}(\Sigma)$-bisimulation} between $\Imc_1$ and $\Imc_2$ if
the following conditions are satisfied for all $(d,e)\in Z$: 
\begin{description}
	\item[Atom] for all concept names $A\in \Sigma$: $d\in A^{\Imc}$ iff $e\in A^{\Jmc}$,
	\item[Back] for all role names $r\in \Sigma$ and all $(d,d')\in r^{\Imc}$, there is $(e,e')\in r^{\Jmc}$ such that $(d',e')\in Z$,

	\item[Forth] for all role names $r\in \Sigma$ and all $(e,e')\in r^{\Jmc}$, there is $(d,d')\in r^{\Imc}$ such that $(d',e')\in Z$.
\end{description}
A \emph{pointed interpretation} is a pair $\Imc,d$ with \Imc an interpretation
and $d\in \Delta^{\Imc}$. For pointed interpretations $\Imc_1,d_1$ and $\Imc_2,d_2$, we write $\Imc_1,d_1
\sim_{\mathcal{ALC},\Sigma}\Imc_1,e_1$ in case $\Imc_1,d_1$ and $\Imc_2,d_2$ are
\emph{$\mathcal{ALC}(\Sigma)$-bisimilar}, that is, if there is an
$\mathcal{ALC}(\Sigma)$-bisimulation $Z$ between $\Imc_1$ and $\Imc_2$ with $(d_1,d_2)\in Z$. 
Bisimulations are a powerful tool since they capture the expressive power of \ALC. Indeed, if $\Imc_1,d_1\sim_{\mathcal{ALC},\Sigma}\Imc_2,d_2$, then $d_1$ and $d_2$ satisfy the same $\ALC(\Sigma)$ concepts, that is, $d_1\in C^{\Imc_1}$ iff $d_2\in C^{\Imc_2}$, for all $\ALC(\Sigma)$ concepts.
The converse direction does not hold in general, but in relevant cases (for example, when $\Imc_1$ and $\Imc_2$ are finite). We refer the interested reader to~\cite{goranko20075} for a more detailed account on model theory. 

%
For the purpose of this chapter, it is worth mentioning that for each extension/restriction \Lmc of \ALC there is a corresponding notion of bisimulation, denoted $\sim_{\Lmc,\Sigma}$, that captures the expressive power of \Lmc. As a concrete example, an $\ALCO(\Sigma)$-bisimulation $Z$ between $\Imc_1$ and $\Imc_2$ is an $\ALC(\Sigma)$-bisimulation that additionally satisfies the following for all $(d,e)\in Z$:
\begin{description}
	\item[AtomI] for all individual names $a\in \Sigma$: $d=a^{\Imc_1}$ iff $e=a^{\Imc_2}$.
\end{description}

\section{Uniform Interpolation for Description Logic Ontologies}\label{sec:ui-global}

\newcommand{\Nui}{\textsc{Nui}\xspace}

The central notion we are concerned with in this section is the following adaptation of Definition~\ref{def:general-interpolants} to description logic ontologies.


\begin{definition}[Uniform interpolant]\label{def:uniformIntDL}
  Let $\Lmc$ be a DL, $\Omc$ an $\Lmc$ ontology, and $\Sigma$ a signature. Then, an $\Lmc$ ontology $\Omc^\Sigma$ is called
  \emph{uniform $\Lmc(\Sigma)$-interpolant} of $\Omc$ if
  \begin{enumerate}
    \item $\Omc\models\Omc^\Sigma$,
    \item $\sig{\Omc^\Sigma}\subseteq\Sigma$,
    \item\label{itm:ui-insep} for every $\Lmc$ CI $\alpha$ such that $\sig{\alpha}\subseteq\Sigma$ and $\Omc\models\alpha$, also $\Omc^\Sigma\models\alpha$.
  \end{enumerate}
\end{definition}
Conditions~1 and~3 imply that $\Omc$ and a uniform $\Lmc(\Sigma)$-interpolant $\Omc^\Sigma$ of $\Omc$ entail precisely the same $\Lmc(\Sigma)$ concept inclusions. This latter condition is called $\Lmc(\Sigma)$-inseparability and of independent interest.
\begin{definition}[Inseparability]\label{def:inseparability}
  Let $\Omc_1$, $\Omc_2$ be \Lmc ontologies and $\Sigma$ a signature. Then, $\Omc_1$ and $\Omc_2$ are
  \emph{$\Lmc(\Sigma)$-inseparable}, in symbols $\Omc_1\equiv_\Sigma^\Lmc\Omc_2$, if 
  $\Omc_1\models\alpha$ iff $\Omc_2\models\alpha$
  for every $\Lmc(\Sigma)$ CI $\alpha$.
\end{definition}

It is immediate from Definition~\ref{def:uniformIntDL} that uniform interpolants are unique modulo logical equivalence, which is why we often speak of \emph{the} uniform $\Lmc(\Sigma)$-interpolant of $\Omc$. In modal logic terminology, this form of interpolation could be classified as \emph{turnstile} interpolation for global consequence. It is \emph{uniform} in the sense that the interpolant has to work for all $\Lmc(\Sigma)$ ontologies that are a consequence of $\Omc$; we refer the reader to~\refchapter{chapter:uniform} for more on uniform interpolation.
%
To illustrate the notion, consider the following example~\cite[Example~2]{FOUNDATIONS_ALC_UI}.
\begin{example}
  Consider the ontology \Omc consisting of the following CIs: 
  \begin{align}
    \textsf{Uni} & \sqsubseteq \exists\textsf{hasEnrolled}.\textsf{Grad} \sqcap \exists\textsf{hasEnrolled}.\textsf{Undergrad}\\
    \textsf{Grad} & \sqsubseteq \neg \textsf{Undergrad} \\
    \textsf{Uni} & \sqsubseteq \neg \textsf{Grad} \\
    \textsf{Uni} & \sqsubseteq \neg \textsf{Undergrad}\label{eq:uni-example}
  \end{align}
  Then, it can be verified that the ontology $\Omc'$ consisting of CI~\eqref{eq:uni-example} and the CI 
  \[ \textsf{Uni} \sqsubseteq \exists\textsf{hasEnrolled}.\textsf{Undergrad} \sqcap \exists\textsf{hasEnrolled}.(\neg \textsf{UnderGrad}\sqcap\neg \textsf{Uni}) \]
  is a uniform $\ALC(\Sigma)$-interpolant of \Omc for $\Sigma=\{\textsf{Uni},\textsf{hasEnrolled},\textsf{Undergrad}\}$. Sometimes $\Omc'$ is called the result of \emph{forgetting} the concept name $\textsf{Grad}$ in \Omc, since $\Omc'$ contains exactly the same information about the signature $\Sigma=\sig{\Omc}\setminus\{\textsf{Grad}\}$.
  \lipicsEnd
\end{example}

Unfortunately, the existence of uniform interpolants is by no means guaranteed, as the following example illustrates already for very simple DL ontologies.
\begin{example}
  Consider the \EL ontology  
 $\Omc=\{A\sqsubseteq B$, $B\sqsubseteq\exists r.B\}$ and $\Sigma=\{A,r\}$. A uniform
$\EL(\Sigma)$-interpolant would need to entail precisely the CIs 
\[
  A\sqsubseteq\exists r.\top,\qquad A\sqsubseteq\exists r.\exists r.\top,\qquad
  A\sqsubseteq\exists r.\exists r.\exists r.\top,\qquad \ldots
\]
which is not possible for a finite set of $\EL$ or \ALC CIs. (Actually, there is not even first-order sentence that entails precisely these CIs, c.f.~\refchapter{chapter:firstorder} for uniform interpolation in first-order logic.)
%
The reader may conjecture that the lack of a uniform interpolant is due to cyclicicty in \Omc, but this is not the entire story.
%
  %
  Consider the \EL ontology 
  \[
    \Omc=\{\ \ A\sqsubseteq\exists r.B,\qquad
              A_0\sqsubseteq\exists r.(A_1\sqcap B),\qquad
              E\equiv A_1\sqcap B\sqcap\exists r.(A_2\sqcap B)\ \ \}
  \]
  and signature $\Sigma=\{A,A_0,A_1,r,E\}$. Reference~\cite{FORGETTING_EL} classifies \Omc as \emph{$\Sigma$-loop free} and shows that hence \Omc has a uniform $\EL(\Sigma)$-interpolant. However, \Omc fails to have a uniform $\ALC(\Sigma)$-interpolant~\cite[Example 3]{FOUNDATIONS_ALC_UI}.
\lipicsEnd 
\end{example}
As we have argued in the introduction, uniform interpolants are useful in a range of applications in KR. Since they do not always exist, it is interesting from a theoretical point of view to investigate existence as a decision problem. From a practical perspective, it is interesting to compute them (in case they exist), study their size, and to deal with the fact that they may not exist. We will address exactly these questions. More precisely, in
Section~\ref{sec:prop-ui}, we discuss the complexity of the existence problem and the size of uniform interpolants. In Section~\ref{sec:ui-practice}, we discuss methods to compute uniform interpolants in practice, addressing also the question what to do if they do not exist. Finally, in Section~\ref{sec:modules}, we discuss further related notions that are relevant to DL specific applications. Throughout the section, we focus on the DLs \EL and \ALC, as these are the ones for which most results are known.

%
%
%
\subsection{Existence and Size}\label{sec:prop-ui}

As motivated in the previous section, 
the focus of this section will be the following \emph{uniform interpolant existence problem} for DL \Lmc: 
\begin{description}
\item [Input] \Lmc ontology $\Omc$, signature $\Sigma$.
\item [Question] Does there exist a uniform $\Lmc(\Sigma)$-interpolant of $\Omc$?
\end{description}
%
The main result known here is that this problem is decidable for both \EL and \ALC, but has complexity one exponential higher than standard reasoning tasks in these DLs.
\begin{theorem}[\cite{COMPLEXITY_UI_EL,FOUNDATIONS_ALC_UI}]\label{thm:uip-complexity}
  Uniform interpolant existence is \ExpTime-complete for \EL and \TwoExpTime-complete for \ALC.
\end{theorem}
We are interested in the size of the uniform interpolants, if they exist, and tight bounds are known here as well. Interestingly, while the complexity of deciding existence is higher for \ALC than for \EL, they exhibit the same bounds on the size of uniform interpolants.
\begin{theorem}[\cite{COMPLEXITY_UI_EL,FOUNDATIONS_ALC_UI,EXP_EXP_EXPLOSION}]\label{thm:uip-sizes}
  Let \Lmc be either \EL or \ALC. 
  \begin{enumerate}
    \item If some \Lmc ontology \Omc has a uniform $\Lmc(\Sigma)$-interpolant for signature $\Sigma$, then there is one of size at most triple exponential in the size of $\Omc$. 
    \item There is a family $\Omc_n$, $n\geq 1$ of \Lmc ontologies and a signature $\Sigma$ such that the size of $\Omc_n$ is polynomial in $n$, a uniform $\Lmc(\Sigma)$-interpolant exists for all $n\geq 1$, but any uniform $\Lmc(\Sigma)$-interpolant has size at least triple exponential in $n$.
  \end{enumerate}
\end{theorem}
It is beyond the scope of this chapter to present all proof details of these results, but we shall give some intuition on how they are obtained. We start with the size lower bounds, that is, Point~2 of Theorem~\ref{thm:uip-sizes}, by recalling the family of \EL ontologies from \cite{EXP_EXP_EXPLOSION} that is used to prove this result. Let $n\geq 1$. Then $\Omc_n$ consists of the following CIs:
%
\begin{align}
  &A_1\sqsubseteq\overline{X_1}\sqcap\ldots\sqcap\overline{X_{n}}\qquad\qquad\qquad\qquad
  A_2\sqsubseteq\overline{X_1}\sqcap\ldots\sqcap\overline{X_{n}} \label{eq:el-init} \\
  &\bigsqcap_{\sigma\in\{r,s\}}\exists\sigma.(\overline{X_i}\sqcap\overline{X_j})\sqsubseteq\overline{X_i} \qquad\qquad
  \bigsqcap_{\sigma\in\{r,s\}}\exists\sigma.(X_i\sqcap\overline{X_j})\sqsubseteq X_i \qquad
  1\leq j<i\leq n \label{eq:el-keep} \\
  &\bigsqcap_{\sigma\in\{r,s\}}\exists\sigma.(
  \overline{X_i}\sqcap X_{i-1}\sqcap\ldots\sqcap X_1)\sqsubseteq X_i
  \qquad 1\leq i\leq n \label{eq:el-flip-1} \\
  &\bigsqcap_{\sigma\in\{r,s\}}\exists\sigma.(
  X_i\sqcap X_{i-1}\sqcap\ldots\sqcap X_1)\sqsubseteq \overline{X_i}\qquad
  1\leq i\leq n \label{eq:el-flip-2} \\
  &X_0\sqcap\ldots\sqcap X_{n}\sqsubseteq B \label{eq:el-finish}
\end{align}
%
Intuitively, $\Omc_n$ constructs an $n$-bit binary counter via the concept names $X_i$,
$\overline{X_i}$ for $1\leq i\leq n$. The satisfaction of these names encodes a number
between $0$ and $2^n-1$ using $n$ bits: satisfaction of $X_i$ means that the $i$th bit
has value $1$, and $\overline{X_i}$ represents a value of $0$ at position~$i$.
Using this, elements can be assigned a counter value $k$.
Specifically, the CIs in \eqref{eq:el-init} make sure that instances of $A_1$ and $A_2$ have a
counter value of 0 (all bits are 0). CIs~\eqref{eq:el-keep}--\eqref{eq:el-flip-2} make sure that,
if both an $r$ and an $s$-successor of an element $d$ has a counter value of $k$, then $d$ has a counter value
of $k+1$. This is done by specifying how the bits on $d$ should be set depending on the bits in the
successors. Specifically, the CIs in~\eqref{eq:el-keep} describe the situation where the bit value
at position $i$ should remain the same (at least one lower bit in the successor has value 0).
CIs~\eqref{eq:el-flip-1} and~\eqref{eq:el-flip-2} describe
that a bit at position $i$ should flip if all lower bits in the successor have a value of 1.
Finally, the CI~\eqref{eq:el-finish} states that elements with a counter value of $2^{n-1}$
(all bits are 1) must satisfy $B$.

This construction has the following effect: any element $d$ that is the root of a full binary $r/s$-tree of depth $2^n$ whose leaves satisfy $A_1$ or $A_2$, will be
an instance of $B$.
For $\Sigma=\{A_1,A_2,B,r,s\}$, any uniform $\Sigma$-interpolant will have to preserve this behaviour without using the counter encoded with the concept names $X_i,\overline X_i$.
To construct the uniform $\EL(\Sigma)$-interpolant for $\Omc_n$, we define inductively sets
$\Cmc_i$ of concepts by setting $\Cmc_0=\{A_1,A_2\}$ and
$\Cmc_{i+1}=\{\exists r.C\sqcap\exists s.C\mid C\in\Cmc_i\}$ for $i\geq 0$.
It is not hard to see that the cardinality of each $\Cmc_i$ is $2^{2^i}$. Moreover,
for each $C\in\Cmc_{2^{n}}$, $\Omc_n\models C\sqsubseteq B$, and we cannot capture these
entailments more concisely than by adding all those $C\sqsubseteq B$ to the uniform interpolant.
It follows that the uniform interpolant contains at least $2^{2^{2^{n}}}$ CIs, and is thus of
size triple exponential in $n$.

This finishes the proof sketch for \EL. The same construction does not directly yield a triple-exponential lower bound for \ALC, since we can here use the single concept $A_1\sqcup A_2$ in case of the two concepts in $\Cmc_0$, which reduces the size of the uniform interpolant to double exponential.
%
The triple exponential lower bound on the size of uniform interpolants for \ALC from~\cite{FOUNDATIONS_ALC_UI,DID_I_DAMAGE}, 
is shown similarly,
but relies on a $2^n$-bit counter instead of an $n$-bit counter.

\bigskip
We now return to the existence problem, this time concentrating on \ALC. Let us fix an \ALC ontology $\Omc$ and a signature $\Sigma$ as input.
We first give a syntactic characterization for the existence of a uniform interpolant. For $n\geq 0$, define $\Omc^\Sigma_n$ as the set of all $\ALC(\Sigma)$ CIs entailed by $\Omc$ whose role depth is at most $n$.
If there exists a uniform interpolant, it must be equivalent to~$\Omc^\Sigma_n$ for some~$n$, since all its CIs would be contained in $\Omc^\Sigma_n$, and conversely $\Omc\models\Omc^\Sigma_n$ by construction. This means that non-existence of uniform interpolants implies that for every~$n$, there exists some $k>n$ such that $\Omc^\Sigma_n\not\models\Omc^\Sigma_k$. 

We characterize when some $\Omc^\Sigma_n$ is (not) a uniform interpolant using tree interpretations, which are interpretations $\Imc$ for which the relation
$\bigcup\{r^\Imc\mid r\in\NR\}$ forms a directed tree. The root of such an interpretation is then the root of that directed tree. For an integer $m$, we use $\Imc^m$ to refer to the interpretation obtained by removing all individuals which are not reachable from the root via a role path of length at most~$m$.
We then have the following lemma.
\begin{lemma}\label{lem:characterization-uis}
 Let $\Omc$ be an ontology, $\Sigma$ a signature and $n>0$. Then, $\Omc_n^\Sigma$ is not a uniform $\ALC(\Sigma)$-interpolant of $\Omc$ if there are tree interpretations $\Imc_1$, $\Imc_2$ with root $d$ such that
 \begin{enumerate}
  \item $\Imc_1^n=\Imc_2^n$,
  \item $\Imc_1,d\sim_{\ALC,\Sigma}\Jmc,d'$ for some model $\Jmc$ of $\Omc$ and $d'\in\Delta^\Jmc$,
  \item $\Imc_2,d\not\sim_{\ALC,\Sigma}\Jmc,d'$ for all models $\Jmc$ of $\Omc$ and $d'\in\Delta^\Jmc$, and
  \item for all role successors $e$ of $d$, $\Imc_2,e\sim_{\ALC,\Sigma}\Jmc,d'$ for some model $\Jmc$ of $\Omc$ and $d'\in\Delta^\Jmc$.
 \end{enumerate}
\end{lemma}
Intuitively, these conditions express that one has to look at a role depth beyond $n$ to capture all~$\Sigma$ entailments of $\Omc$, or, equivalently, to capture all pointed interpretations that are $\Sigma$-bisimilar to some pointed model of $\Omc$. Together with Condition~2, Condition~1 ensures that the root individual in both pointed interpretations satisfies the same CIs up to role depth $n$, which means they satisfy all CIs in $\Omc_n^\Sigma$. However, $\Imc_2,d$ is not $\Sigma$-bisimilar to all pointed models of $\Omc$, while
$\Imc_1$ is (Condition~3). Finally, Condition~4 states that
the root is the point where the bisimulation breaks, since in $\Imc_2$, all successors of $d$ are $\Sigma$-bisimilar to pointed models of $\Omc$ as required. In other words, $\Imc_2,d$ witnesses precisely that $\Omc_n^\Sigma$ is not a uniform $\Sigma$-interpolant of~$\Omc$.

\newcommand{\Ext}{\textsf{Ext}}
\newcommand{\bound}{M_\mathcal{O}}

Set $\bound=2^{2\cdot2^{\size{\Omc}}}+1$. We can show that if Lemma~\ref{lem:characterization-uis} applies for $n=\bound$, then it also applies for all $n>\bound$, meaning that then there does not exist a uniform $\ALC(\Sigma)$-interpolant. To do this, we use a construction based on types.
Let $\Gamma$ denote the set of all subconcepts that occur in \Omc, closed under single negation.
Given an interpretation $\Imc$ and some individual $d\in\Delta^\Imc$,
we then define the \emph{type $\tp_\Imc(d)$ of $d$ in $\Imc$} as
\[
    \tp_\Imc(d)=\{C\in\Gamma\mid d\in C^\Imc\}.
\]
To capture the types of elements that are bisimilar to $d$, we define the \emph{extension set of $d$} as
\[
 \Ext(d,\Imc)=\{\tp_\Jmc(d')\mid \Jmc\models\Omc, d'\in\Delta^\Jmc \text{ and }\Jmc,d'\sim_\Sigma\Imc,d \}
\]
There are at most $2^{2^{\size{\Omc}}}$ extension sets. Now assume Lemma~\ref{lem:characterization-uis} applies for $n=\bound$, and let $\Imc_1$ and $\Imc_2$ be the witnessing tree interpretations. We have $\Imc_1^n=\Imc_2^n$, and now look at the successors in the leafs of $\Imc_1^n$, i.e. the elements in $\Imc_1^n$ without a role successor. Specifically, for such a pair, we define the set $D_n(\Imc_1,\Imc_2)$ to contain all leafs in $\Imc_1^n$ which differ in $\Imc_1^{n+1}$ and $\Imc_2^{n+1}$ in their sets of role successors, by which we mean that they either have a different set of role successors or some role successor that satisfies a different set of concept names in one interpretation compared to the other. Fix a pair $\Imc_1$, $\Imc_2$ of interpretations that satisfy the conditions in Lemma~\ref{lem:characterization-uis} for $n=\bound$. If $D_n(\Imc_1,\Imc_2)=\emptyset$, then lemma also applies for $n=\bound+1$, and we are done. Otherwise, we pick some $f\in D_n(\Imc_1,\Imc_2)$, and consider the path $d=e_1$, $e_2$, $\ldots$, $e_{n}=f$ of elements such that, for $1\leq i<n$, $\tup{e_i,e_{i+1}}\in r_i^{\Imc_1}$ for some role name $r_i$. Because $n=\bound=2^{2\cdot2^{\size{\Omc}}}+1$, and there are at most $(2^{2^{\size{\Omc}}})^2=2^{2\cdot 2^{\size{\Omc}}}$ pairs of extension sets, there must be some $1\leq i<j\leq n$ such that  $\Ext(e_i,\Imc_1)=\Ext(e_j,\Imc_1)$ and $\Ext(e_i,\Imc_2)=\Ext(e_j,\Imc_2)$. In both $\Imc_1$ and $\Imc_2$, we now replace the subtree below $e_j$ with a copy of the subtree below $e_i$, resulting in two new interpretations $\Kmc_1$ and $\Kmc_2$. One can show that this construction preserves the conditions in Lemma~\ref{lem:characterization-uis}, so that $\Kmc_1$ and $\Kmc_2$ also witness it. In addition, we have $D_n(\Kmc_1,\Kmc_2)\subsetneq D(\Imc_1,\Imc_2)$. We can repeat this operation until we obtain a pair $\Jmc_1$, $\Jmc_2$ of interpretations for which
$D_n(\Jmc_1,\Jmc_2)=\emptyset$, and thus $\Jmc_1^{n+1}=\Jmc_2^{n+1}$, and obtain that Lemma~\ref{lem:characterization-uis} also holds for $n+1$. This gives us the following lemma.
\begin{lemma}\label{lem:ui-bound}
 There is a uniform $\Lmc(\Sigma)$-interpolant of \Omc iff $\Omc^\Sigma_n\models\Omc^\Sigma_{n+1}$ for $n=\bound$.
\end{lemma}

A consequence of lemma~\ref{lem:ui-bound} is that if a uniform interpolant exists, then $\Omc_n^\Sigma$ is one, where $n=\bound$. Thus, it gives a bound on the role depth of uniform interpolants. Since, up to equivalence, there are only finitely many $\ALC(\Sigma)$ CIs of bounded role depth, this also shows that the existence problem is decidable. Unfortunately, the bounds we get in this way are non-elementary. To establish the \TwoExpTime-upper bound claimed in Theorem~\ref{thm:uip-complexity},~\cite{FOUNDATIONS_ALC_UI} uses an automata-based approach to decide existence of uniform interpolants. The idea is to construct a tree automaton that decides the existence of interpretations $\Imc_1$ and $\Imc_2$ that satisfy the conditions in Lemma~\ref{lem:characterization-uis} for $n=\bound$. In the same paper, they show the triple exponential upper bound on the size of uniform \ALC-interpolants claimed in Point~1 of Theorem~\ref{thm:uip-sizes}. The proof is by a reduction of the problem of computing uniform interpolants of ontologies to the problem of computing uniform interpolants of \emph{concepts} and apply known results for the latter~\cite{TenEtAl06}. For the complexity lower bounds, we refer the reader to~\cite{FOUNDATIONS_ALC_UI} as well.


Interestingly, the strategy for deciding existence of uniform interpolants for \EL is similar~\cite{COMPLEXITY_UI_EL}: it relies on a characterization in the style of Lemma~\ref{lem:characterization-uis} (using simulations instead of bisimulations), and shows a bound on the role depth of uniform interpolants in style of Lemma~\ref{lem:ui-bound} (which is single exponential in this case). Automata are then used as well, though in a different way, to obtain the complexity upper bounds claimed in Theorem~\ref{thm:uip-complexity}. The size bounds for \EL in Theorem~\ref{thm:uip-sizes} are shown in~\cite{EXP_EXP_EXPLOSION}.

\subsection{Computing Uniform Interpolants in Practice}\label{sec:ui-practice}

\begin{table}
 \centering
 \begin{tabular}{l|l|l|l}
  Name of Tool & Logic & Technique & Dealing with Non-Existence  \\ 
  \hline
  \Nui~\cite{FORGETTING_EL} & \EL & TBox unfolding & -- \\
  -~\cite{PRACTICAL_UI_LUDWIG}  & \ALC & resolution & approximation \\
  \Lethe~\cite{LETHE}  & \ALC & resolution + Ackermann's Lemma & fixpoints/auxiliary symbols \\
  \Fame~\cite{FAME} & \ALC & Ackermann's Lemma & fixpoints/auxiliary symbols
 \end{tabular}
 \caption{Overview of tools for computing uniform interpolants of ontologies in practice.
 }
 \label{tab:ui-tools}
\end{table}


Despite these high complexity bounds, algorithms for computing uniform interpolants have been described
for many description logics,
ranging from light-weight DLs such as DL-Lite~\cite{UI_DL_LITE}
and \EL~\cite{FORGETTING_EL,LUDWIG_UI_EL,LiuLASZ21} to \ALC and more expressive
DLs~\cite{WANG,PRACTICAL_UI_LUDWIG,LETHE,FAME,PHD_THESIS}.
Many of these were implemented (see \Cref{tab:ui-tools}), and are able to compute results even for
larger realistic ontologies. In this section, we discuss how implemented methods
compute uniform interpolants of \ALC ontologies, describing the method used by \Lethe in detail.

The first practical issue is of course that uniform interpolants do not always exist.
In many cases, a pragmatic solution could be to
just extend
the desired signature by the problematic symbols.
For the case where this is not an option,
there are three solutions, all of which are implemented by some tools.

\begin{description}
\item[Option 1:] We approximate the uniform interpolant up to a given role depth, which might be sufficient
if we are only interested in entailments of bounded role depth.
The resolution based approach proposed in~\cite{PRACTICAL_UI_LUDWIG} provides guarantees on the role depth of preserved entailments---unfortunately these
guarantees require an exponential increase in the role depth in the uniform interpolant.

\item[Option 2:] We extend the DL by \emph{greatest fixpoint operators}.
As argued above, for the ontology $\Omc=\{A\sqsubseteq B, B\sqsubseteq\exists r.B\}$,
a uniform $\ALC(\{A,r\})$-interpolant does
not exist. However, there does exist one in $\ALC\mu$:
$\Omc_\Sigma=\{A\sqsubseteq\nu X.\exists r.X\}$. Indeed, $\ALC\mu$ does have the
uniform interpolation property,
which follows from the corresponding result for the modal
$\mu$-calculus~\cite{MU-CALCULUS}; see also \refchapter{chapter:cyclic}.
%
$\ALC\mu$ extends $\ALC$ by concepts of the form $\nu X.C[X]$, where $X$ is taken from
a set $\NV$ of \emph{concept variables} that is pair-wise disjoint with $\NC$, $\NR$ and $\NI$,
and $C[X]$ is a concept in which $X$ is used like a concept name, but occurs
only positively, that is, under an even number of negations.
For a formal definition of the semantics, we refer to~\cite{DBLP:conf/ijcai/CalvaneseGL99}.
Intuitively, $\nu X.C[X]$ corresponds to the limit of the sequence
\[ \top,\quad C[\top],\quad C[C[\top]],\quad C[C[C[\top]]], \ldots, \]
where in each case, $C[D]$ refers to the result of replacing $X$ in $C[X]$ by $D$.

\item[Option 3:] Since fixpoint operators are harder to
understand for end-users and also not supported by the OWL standard, which is the format used to store ontologies
in practical applications,
a third option is to simulate greatest fixpoint
operators using auxiliary concepts: In particular, we can replace any greatest fixpoint expression
$\nu X.C[X]$, provided that it occurs only positively (e.g., under an even number of negations on the left-hand side
of a CI, and under an odd number of negations on the right-hand side),
by a fresh concept name $D$ for which we add a new CI $D\sqsubseteq C[D]$.
We then approximate the uniform $\Sigma$-interpolant \emph{signature-wise}, resulting in an ontology that is
$\Sigma$-inseparable from the original ontology and uses
names outside of $\Sigma$ in a syntactically restricted way.
\end{description}
Regardless of the option chosen, the general idea of practical uniform interpolation methods is
to eliminate the names that should not occur in the
interpolant using dedicated inferences.
For this, we find two approaches: 
1) using resolution~\cite{PRACTICAL_UI_LUDWIG}
(as in \refchapter{chapter:propositional} and \refchapter{chapter:automated}))
or 2) using Ackermann's Lemma~\cite{Ackermann,FAME} (also discussed in \refchapter{chapter:automated}).
The method in \cite{LETHE}, which we present in the following, uses both in combination.

Because \ALC is more complex than propositional logics, and syntactically more restricted than
first-order logics, in order to pursue 1), we cannot use standard resolution methods, but need a specific resolution procedure that is appropriate for the logic and the aim of
computing uniform interpolants. The first resolution-based approach for uniform interpolation in \ALC was presented in~\cite{PRACTICAL_UI_LUDWIG} and is based on a complex inference system for modal logic introduced in~\cite{EnjalbertC89} and extended for uniform interpolation in modal logics in~\cite{HerzigM08}.
This inference system intuitively specifies an unbounded number of inference rules
that allow to perform resolution on symbols occurring in different nesting levels of role restrictions inside a CI.
By increasing the bound on the nesting level on which inferences are performed,
we can then approximate the uniform interpolant up to a specified role depth.

\added{
Ackermann's Lemma~\cite{Ackermann} has been extensively used in the context of
Second-Order Quantifier Elimination.
To formulate it for DLs, we use a stronger version of inseparability as introduced in Definition~\ref{def:inseparability}, in which instead of \Lmc concept inclusions arbitrary second-order logic (\SO) sentences are considered. We use $\Omc_1\equiv_\Sigma ^\SO \Omc_2$ to denote that $\Omc_1,\Omc_2$ entail the same \SO sentences over signature $\Sigma$.
Let $\Omc[A\mapsto C]$ denote the ontology that is obtained from $\Omc$ by replacing
every occurrence of $A$ by $C$.}
\begin{lemma}\label{lem:ackermann}
  Let $\Omc$ be an ontology, $A$ a concept name, $C$ a concept with $A\not\in\sig{C}$,
  and
  $\Sigma=\sig{\Omc}\setminus\{A\}$. Assume $A$ occurs
  only positively in $\Omc$. Then, $\Omc\cup\{A\sqsubseteq C\} \equiv_\Sigma^\SO\Omc[A\mapsto C]$.
\end{lemma}
Intuitively, this means that, if we can bring the ontology into the right form, we can compute a very strong kind of
uniform interpolants by simply applying the equivalence in Lemma~\ref{lem:ackermann} left-to-right.
%
If $A$ does occur on the right-hand side in $A\sqsubseteq C$, we can use a generalization
of Ackermann's lemma from~\cite{nonnengart1999elimination} that introduces greatest fixpoints.
When restricted to DLs, it can be formulated as follows.\footnote{\added{For the original FO formalization of both lemmas and their application to uniform interpolation, see also
  \refchapter{chapter:automated}.}
}
\begin{lemma}[\added{Generalized Ackermann's Lemma}]\label{lem:gen-ackermann}%
  Let $\Omc$ be an ontology, $A$ a concept name, $C[A]$ a concept in which $A$ occurs,
  and
  $\Sigma=\sig{\Omc}\setminus\{A\}$. Assume $A$ occurs
  only positively in $\Omc$ and $C[A]$. Then, the following holds: 
  \[
    \Omc\cup\{A\sqsubseteq C[A]\} \equiv_\Sigma^\SO\Omc[A\mapsto \nu X.C[X]]
  \]
\end{lemma}

\added{For both lemmas to be applicable, we have to bring the ontology into the required form.}
In particular, one has to isolate all negative
occurrences of $A$ into a single axiom of the form $A\sqsubseteq C$, which is not always possible.
The approach in~\cite{FAME}, implemented in the uniform interpolation tool \Fame, essentially follows
this strategy, and uses additional constructs such as inverse roles to make this possible, for instance, using the equivalence of axioms $C\sqsubseteq\forall r.D$ and $\exists r^-.C\sqsubseteq D$.
The full method implemented in \Fame is much more complex, and cannot only eliminate concept names, but also
role names. The advantage of using (generalized) Ackermann's lemma is that the computed uniform interpolants
are of a very strong kind, namely they are $\SO(\Sigma)$-inseparable.
This is why the method implemented in \Fame is also called \emph{semantic forgetting}.
Since $\SO(\Sigma)$-inseparability is undecidable (see \Cref{sec:modules}), it is in general not
possible to perform semantic forgetting to obtain an ontology expressed in a decidable logic, and indeed
\Fame cannot compute uniform interpolants for all possible inputs.

\newcommand{\ND}{\mathsf{N_D}}

The technique in~\cite{LETHE}, implemented in the uniform interpolation tool \Lethe, combines resolution with Ackermann's lemma and is based on a special normal form for \ALC ontologies.
%
%
In this normal form, axioms are represented as disjunctions $L_1\sqcup\cdots\sqcup L_n$ of literals of the following forms, where
$A\in\NC$ and $D$ belongs to special set $\ND\subseteq\NC$ of concept names called \emph{definers}:
\[
  A \quad\mid\quad \neg A \quad\mid\quad \exists r.D \quad\mid\quad\forall r.D
\]
Such disjunctions are called \emph{DL clauses}, and are interpreted \emph{globally}.
For example, the concept inclusion $A\sqsubseteq B$ is represented as the DL clause
$\neg A\sqcup B$. \ALC ontologies $\Omc$ are normalized into $\sig{\Omc}$-inseparable ontologies in this normal form through standard
CNF transformation techniques and through the introduction of fresh concept names which we call \emph{definers}.
For instance, we would replace an axiom
$A\sqsubseteq\exists r.(B\sqcup C)$ by two clauses $\neg A\sqcup\exists r.D$, $\neg D\sqcup B\sqcup C$, where
$D$ is an introduced definer.
These definers play a central role in the method. We can normalize ontologies so that
every clause contains
at most one negative occurence of a definer. This invariant ensures that for every definer $D$, we can transform the
clauses that contain $\neg D$ into a single GCI of the form $D\sqsubseteq C$, and then eliminate $D$ again
by using Ackermann's lemma or its generalization. The resolution procedure produces
only clauses that also satisfy this invariant, which is why all definers can be eliminated in a final step.

\newcommand{\quant}{\mathsf{Q}}
\newcommand{\Scan}{\textsc{Scan}}

The resolution method used by \Lethe uses the following two inference rules:
\[
  \textbf{Resolution: }\dfrac{C_1\sqcup A\quad C_2\sqcup\neg A}{C_1\sqcup C_2}
  \hfill
  \textbf{Role Propagation: }\dfrac{C_1\sqcup\forall r.D_1\qquad C_2\sqcup\quant r.D_2}{C_1\sqcup C_2\sqcup\quant r.D_{12}}
\]
where $\quant\in\{{\exists},{\forall}\}$, and $D_{12}$ is a possibly fresh definer that represents
$D_1\sqcap D_2$. We can introduce such a definer by adding the clauses
$\neg D_{12}\sqcup D_1$, $\neg D_{12}\sqcup D_2$ and immediately resolving on $D_1$ and $D_2$.
In order to ensure termination, the method keeps track of introduced definers and reuses them where
possible---for details we refer to~\cite{LETHE}.

The idea is now to perform all possible inferences on the concept names we want to eliminate, where
we avoid inferences that would introduce a clause with two negative occurences of a definer.
Because of this
invariant, some resolution inferences
become only possible through application of the role propagation rule, as illustrated in the following example.
Once all inferences on a concept name to be eliminated have been applied, we can remove all
occurences of that concept name.\footnote{This is essentially the same idea as in the resolution-based approach
for propositional logic discussed in \refchapter{chapter:propositional},
or in the second-order quantifier elimination tool \Scan~\cite{SCAN} described in
\refchapter{chapter:automated}.}
Once we have eliminated all concept names outside of the signature $\Sigma$ in
this way, we eliminate all introduced definers using Ackermann's lemma and its generalization to obtain the
uniform interpolant.
To eliminate a role name $r$, we perform all possible inferences on $r$ using role propagation. Afterwards,
we can remove literals $\exists r.D$ for which $D$ is unsatisfiable, and filter out the remaining occurrences of
$r$.\footnote{\cite{MyRoleForgetting} discusses a more involved method for role elimination. Both methods are implemented in the current version of \Lethe~\cite{DEDUCTIVE_MODULES,LETHE}.}

\begin{example}
  We want to compute the $\{A,B,D,E,r\}$-interpolant of the following ontology:
  \[
   \Omc=\{\quad A\sqsubseteq\exists r.(B\sqcap C)\qquad \exists r.(C\sqcap D)\sqsubseteq E\quad \}
  \]
  The corresponding set of DL clauses is the following:
  \[
   1.\ \neg A\sqcup\exists r.D_1\qquad
   2.\ \neg D_1\sqcup B\qquad
   3.\ \neg D_1\sqcup C\qquad
   4.\ \forall r.D_2\sqcup E\qquad
   5.\ \neg D_2\sqcup\neg C\sqcup\neg D
  \]
  We need to eliminate $C$. We cannot resolve on Clauses~3 and~5, since it would produce a clause with two negative occurrences of a definer, which would break our invariant. To make resolution on $C$ possible, we need to first apply role propagation on Clauses~1 and 4, followed by an immediate resolution on $D_1$ and $D_2$:
  \begin{align*}
   6.&\ \neg A\sqcup E\sqcup\exists r.D_{12}\qquad
   &7.&\ \neg D_{12}\sqcup D_1\qquad
   &8.&\ \neg D_{12}\sqcup D_2\\
   9.&\ \neg D_{12}\sqcup B\qquad
   &10.&\ \neg D_{12}\sqcup C
   &11.&\ \neg D_{12}\sqcup \neg C\sqcup\neg D
  \intertext{
  Clauses~7 and~8 are only intermediate inferences and can be discarded. We can now resolve upon $C$, the name to be eliminated, namely on Clauses~10 and~11:
  }
  12.&\ \neg D_{12}\sqcup\neg D
 \end{align*}
 No further resolution steps on $C$ are possible without breaking the invariant that every clause contains at most one negative occurence of a definer.
 The procedure would now remove all clauses containing $C$ (Clause~3, 5, 10 and 11) and apply Ackermann's lemma to eliminate the introduced names $D_1$, $D_2$ and $D_{12}$ again. The resulting set of DL clauses can be represented as the following set of CIs, which is the desired uniform interpolant:
 \[
  \{\quad A\sqsubseteq\exists r.B,\qquad A\sqcap\forall r.(\neg B\sqcup D)\sqsubseteq E\quad \}\lipicsEnd
 \]
\end{example}


Note that while this approach will always compute a uniform interpolant in $\ALC\mu$,
it may not find a uniform interpolant without fixpoints even if one exists. As argued above, in practice, one can keep the
problematic definers, or compute approximations of the fixpoint expressions.
A practical method that is able to decide the existence of uniform interpolants without fixpoints on realistic ontologies
has not been found so far. 

\subsection{Related Notions and Applications}\label{sec:modules}

We close this section by discussing four relevant notions closely related to uniform interpolation: inseparability and conservative extensions, logical difference, and modules.

\medskip

\textbf{Inseparability and Conservative Extensions. }
Inseparability as in Definition~\ref{def:inseparability} is a key property that is desired in many applications. 
For instance,
when an ontology evolves by adding, removing, or changing concept inclusions, we might want to ensure that the change does not affect entailments in a certain signature of interest.
Since inseparability is built-in into uniform interpolation, there are natural connections to these applications; we discuss some of them.

Very closely
related to uniform interpolation is the notion of conservative extensions. Conservative extensions are a classical notion studied in logic which has been introduced to KR in the seminal work~\cite{DID_I_DAMAGE}.
Let \Lmc be a DL and $\Omc_1$, $\Omc_2$ be two \Lmc ontologies. Recall that we write $\Omc_1\equiv^\Lmc_\Sigma\Omc_2$ to indicate that $\Omc_1$ and $\Omc_2$ entail precisely the same concept inclusions that can be expressed in $\Lmc$ using signature $\Sigma$. We call $\Omc_2$ a \emph{conservative extension} of $\Omc_1$ if $\Omc_1\subseteq\Omc_2$ and $\Omc_1\equiv^\Lmc_{\sig{\Omc_1}}\Omc_2$. Conservative extensions are relevant when we want to extend an ontology with new concept inclusions about new symbols, while preserving the behaviour with respect to the original signature. Hence, the central problem is to decide given $\Omc_1,\Omc_2$ whether $\Omc_2$ is a conservative extension of $\Omc_1$. As pointed out in~\cite{DID_I_DAMAGE}, it is immediate from the definitions that
\begin{equation}
\text{$\Omc_2$ is a conservative extension of $\Omc_1$\ \ iff\ \ $\Omc_1$ is a uniform $\Lmc(\sig{\Omc_1})$-interpolant of $\Omc_2$.} \label{eq:cons-ui}
\end{equation}
%
This can in some cases be exploited to decide conservative extensions. Indeed, if $\Omc_2$ has a uniform $\Lmc(\sig{\Omc_1})$-interpolant, then we can compute it and check equivalence with $\Omc_1$. As discussed before, this is however not always possible since uniform interpolants may not exist. In general, even when not only adding concept inclusions, we may want to decide whether two ontologies are inseparable. As for conservative extensions, uniform interpolants could help also in this case. Indeed, if $\Omc_1^\Sigma$ and $\Omc_2^\Sigma$ are uniform $\Lmc(\Sigma)$-interpolants of $\Omc_1$ and $\Omc_2$, respectively, then $\Omc_1\equiv_\Sigma^\Lmc\Omc_2$ iff $\Omc_1^\Sigma\equiv\Omc_2^\Sigma$, that is, $\Omc_1^\Sigma\models \Omc_2^\Sigma$ and $\Omc_2^\Sigma\models \Omc_1^\Sigma$.

\begin{table}
\centering
 \begin{tabular}{l|l}
  Ontology Language & Complexity \\
  \hline\hline
  \EL &  \ExpTime~\cite{DBLP:journals/jsc/LutzW10} \\
  $\ALC$, $\mathcal{ALCI}$ & \TwoExpTime~\cite{DID_I_DAMAGE,CONSERVATIVE_2007} \\
  \ALCO & 3\ExpTime-hard~\cite{kr2021} \\
  $\mathcal{ALCFIO}$ & undecidable~\cite{CONSERVATIVE_2007} \\
 \end{tabular}
 \caption{The complexity of uniform interpolant recognition.
}
 \label{tab:ui-complexities}
\end{table}

Both conservative extensions and inseparability have been heavily studied for description logics, see for example~\cite{DID_I_DAMAGE,CONSERVATIVE_2007,INSEP_SURVEY,CONSERVATIVE_2020,DBLP:journals/jsc/LutzW10,kr2021}. The complexity results obtained there have implications also on the side of uniform interpolation, especially for the recognition problem. Here, \emph{uniform interpolant recognition} is the problem of deciding whether an ontology $\Omc^\Sigma$ is a uniform $\Lmc(\Sigma)$-interpolant of another ontology $\Omc$. The recognition problem has not been explicitly studied in the literature, but it might be relevant for cases when the ontology engineer has a candidate ontology in mind and wants to verify whether it is indeed a uniform interpolant. Table~\ref{tab:ui-complexities} summarizes some complexity results for the recognition problem. The lower bounds are inherited from deciding conservative extensions via the reduction~\eqref{eq:cons-ui}. For the upper bounds, observe that uniform interpolant recognition reduces to entailment and inseparability as follows: 
\[\text{$\Omc^\Sigma$ is a uniform $\Lmc(\Sigma)$-interpolant of $\Omc$\ \ iff\ \ $\Omc\models\Omc^\Sigma$ and $\Omc\equiv_\Sigma^\Lmc\Omc^\Sigma$}.\] 
It is worth noting that nominals
are particularly challenging in this context, and we will see in Section~\ref{sec:interpolation-concept-level} that this is also the case for interpolation at the level of concepts.

We point out that in applications related to query answering, we need alternative notions of inseparability (and thus uniform interpolation), namely inseparability \emph{by queries}. As this is beyond the scope of this chapter, we refer the reader to~\cite{INSEP_SURVEY} for a survey. 

\newcommand{\Star}{\ensuremath{\top\bot*}}


\medskip

\textbf{Logical Difference.}
In case two ontologies are not $\Lmc(\Sigma)$-inseparable, a natural question is in which $\Lmc(\Sigma)$-entailments they actually differ, which is captured by the logical difference of the ontologies:
\newcommand{\logDiff}{\textsf{diff}}

\begin{definition}[Logical Difference]
  Let $\Omc_1$, $\Omc_2$ be two ontologies and $\Sigma$ be a signature. The logical difference from $\Omc_1$ to $\Omc_2$ in
  $\Sigma$ is defined as
  \[
    \logDiff(\Omc_1,\Omc_2,\Sigma)=\{\alpha\mid\sig{\alpha}\subseteq\Sigma,\Omc_1\models\alpha,\Omc_2\not\models\alpha\}
  \]
\end{definition}
Logical difference is particularly useful to track changes between different versions of an ontology. In contrast to a purely syntactic check (which CIs have been added/removed/changed), logical difference provides a semantic way to understand whether entailments
in a signature of interest have changed.
Note that if the logical difference between two ontologies is non-empty, then it contains infinitely many concept inclusions.
Nonetheless, it is often possible to compute finite representations~\cite{KonevWW08}. For more expressive logics, the only existing tools compute finite representations of the logical difference by reduction to
uniform interpolation~\cite{PRACTICAL_UI_LUDWIG,LETHE,LiuLASZ21}.


\medskip

\textbf{Modules.}
As discussed in the introduction, one key application of uniform interpolation is ontology reuse. Modules are an alternative solution for this. Modules work the other way around to conservative extensions: given an ontology $\Omc$ and a signature $\Sigma$, a \emph{$\Lmc(\Sigma)$-module of $\Omc$} is an $\Lmc(\Sigma)$-inseparable subset of $\Omc$.
Modules are investigated in detail in~\cite{GrauHKS08,KonevLWW09,KonevL0W13}, and are relevant for the application \emph{Modularisation and Reuse} discussed in the introduction,
but they have also been used for other purposes such as to improve reasoner performance~\cite{RomeroGH12}, or as preprocessing step for uniform interpolant computation~\cite{PRACTICAL_UI_LUDWIG,LETHE,FAME}.
%
Differently to uniform interpolants, the size of a module is bounded by the size of the original ontology, and the
syntactical structure of CIs is not changed, which is sometimes relevant.
At the same time, uniform interpolants can often be more compact than modules, and give guarantees on the used signature.
Practical methods for uniform interpolation have been used to compute
subset-minimal $\EL(\Sigma)$ and $\ALC(\Sigma)$-modules~\cite{DEDUCTIVE_MODULES,DBLP:conf/aaai/Yang0B23}, but
there are also very fast syntactical methods that compute $\SO(\Sigma)$-modules~\cite{GrauHKS08},
which may however not be subset-minimal.
%
%
A compromise between modules and uniform interpolants that can leverage the advantages of both are
\emph{generalised modules}.
A generalised $\Lmc(\Sigma)$-module of an ontology $\Omc$
can be just any ontology that is $\Lmc(\Sigma)$-inseparable with $\Omc$, but is ideally smaller and simpler than $\Omc$. General modules can be significantly simpler than both uniform interpolants and modules~\cite{SWEET_SPOT,GENERAL_MODULES}. The method in~\cite{GENERAL_MODULES} uses a technique similar to that of \Lethe discussed in this chapter.

\section{Craig Interpolation for Description Logic Concepts}\label{sec:interpolation-concept-level}

In this section, we will define and discuss interpolation in description logics at the level of concepts rather than at the level of ontologies as in preceding section. We start in \Cref{sec:interpolation} with (Craig) interpolation and then move in \Cref{sec:beth-definability} to the tightly related concept of Beth definability. In \Cref{sec:applications}, we discuss the main applications of interpolants and explicit definitions which motivate the necessity of computing them. The computation problem is then covered in \Cref{sec:interpolant-computation}. 

\subsection{Craig Interpolation}\label{sec:interpolation}

We use the following standard definition of
interpolants in the context of DL ontologies.

\begin{definition}[Interpolants]
  Let $\Lmc$ be any DL. Let $C_1,C_2$ be $\Lmc$ concepts,
  $\Omc$ an $\Lmc$ ontology and $\Sigma$ be a signature. Then, an \emph{$\Lmc(\Sigma)$-interpolant for $C_1$ and $C_2$ under $\Omc$} is any $\Lmc(\Sigma)$ concept $I$ that satisfies $\Omc\models C_1\sqsubseteq I$ and $\Omc\models I\sqsubseteq C_2$.
\end{definition}
%
We remark that this definition of interpolants is in line with the one used for modal logic in \refchapter{chapter:sixproofs}. More precisely, with an empty ontology $\Omc$, an interpolant for $C_1$ and $C_2$ is an interpolant for the validity $\models C_1\to C_2$, when viewing $C_1,C_2$ as modal logic formulas. The case with ontologies is also related to the setting of interpolation \emph{relative to theories} in first-order logic, see \refchapter{chapter:firstorder}. 

A classical application of interpolants in DLs is in explaining subsumption relations. More precisely, the idea put forward in~\cite{DBLP:conf/jelia/Schlobach04} is that an $\Lmc(\Sigma)$-interpolant for $C$ and $D$ with $\models C\sqsubseteq D$ for a \emph{minimal} $\Sigma$ is a good explanation of the possibly complicated subsumption $C\sqsubseteq D$. The following example illustrates the idea. 
\begin{example}\label{ex:interpolant-explanation}
Consider $C=\exists \textsf{child}.\top\sqcap \forall \textsf{child}.\textsf{Doctor}$ and $D=\exists\textsf{child}.(\textsf{Doctor}\sqcup \textsf{Rich})$. Clearly, $\models C\sqsubseteq D$ and $\exists \textsf{child}.\textsf{Doctor}$ is an $\ALC(\{\textsf{child},\textsf{Doctor}\})$-interpolant of $C$ and $D$. Moreover, one can show that there is no $\ALC(\Sigma)$-interpolant of $C$ and $D$ for $\Sigma\subsetneq \{\textsf{child},\textsf{Doctor}\}$. Hence, $\exists \textsf{child}.\textsf{Doctor}$ is an arguably simple explanation of the subsumption. \lipicsEnd 
\end{example}
This application motivates the study of the following \emph{\Lmc-interpolant existence} problem:
\begin{description}
\item [Input] \Lmc ontology $\Omc$, \Lmc concepts $C_1,C_2$ and a signature $\Sigma$.
\item [Question] Does there exist an $\Lmc(\Sigma)$-interpolant for $C_1$ and $C_2$ under $\Omc$?
\end{description}
We will see that this potentially difficult existence problem has a surprisingly simple solution for DLs that satisfy the \emph{Craig interpolation property} (CIP), defined following~\cite{TenEtAl13,DBLP:journals/tocl/ArtaleJMOW23}. Recall that $\sig{X}$ denotes the set of concept and role names used in object $X$; we use $\sig{\Omc,C}$ to abbreviate $\sig{\Omc}\cup \sig{C}$. 
\begin{definition}[Craig Interpolation Property]
  Let \Lmc be any DL. We say that \Lmc enjoys the \emph{Craig
  interpolation property (CIP)} if for every
  \Lmc concepts $C_1,C_2$ and \Lmc ontologies $\Omc_1,\Omc_2$ with $\Omc_1\cup\Omc_2\models C_1\sqsubseteq
  C_2$, there exists a \emph{Craig interpolant}, that is, an $\Lmc(\Sigma)$-interpolant of $C_1$ and $C_2$ under
  $\Omc_1\cup\Omc_2$ where $\Sigma = \sig{\Omc_1,C_1}\cap\sig{\Omc_2,C_2}$.
\end{definition}
%
%
In the same way as the definition of interpolants, the definition of the CIP under empty ontologies $\Omc_1=\Omc_2=\emptyset$ coincides with the standard definition of the CIP in modal logic. 
The split of the ontology into two parts is necessary to ensure the existence of Craig interpolants in the presence of ontologies. Indeed, consider the following example from~\cite{DBLP:journals/tocl/ArtaleJMOW23}. Let $\Omc=\{A_1\sqsubseteq A_2,A_2\sqsubseteq A_3\}$, and consider concepts $C_1=A_1$ and $C_2=A_3$. Then for every $\Omc_1,\Omc_2$ with $\Omc=\Omc_1\cup\Omc_2$, there is an $\ALC(\Sigma)$-interpolant for $C_1$ and $C_2$ under $\Omc$ with $\Sigma =  \sig{\Omc_1,C_1}\cap\sig{\Omc_2,C_2}$.
However, there is no $\ALC(\Sigma_0)$-interpolant for $C_1$ and $C_2$ under $\Omc$ with $\Sigma_0=\sig{C_1}\cap\sig{C_2}=\emptyset$.

The Craig interpolation property is so powerful since it \emph{guarantees} the existence of interpolants/explanations in terms of the common signature $\Sigma=\sig{\Omc_1,C_1}\cap\sig{\Omc_2,C_2}$. Fortunately, many DLs enjoy the CIP, also in the presence of ontologies~\cite{TenEtAl13}.
\begin{theorem}\label{thm:cip}
  \ALC and any extension with $\Smc$, $\Imc$, $\Fmc$ enjoys the Craig interpolation property.
\end{theorem}
The proof of Theorem~\ref{thm:cip} in~\cite{TenEtAl13} is based on a tableau algorithm for subsumption and is constructive in the sense that it also computes a Craig interpolant in case the subsumption holds.
We will next provide a shorter
(though not constructive) model-theoretic proof, which bears a lot of similarity
with the model-theoretic proof for the CIP in modal logic in \refchapter{chapter:sixproofs}. The main difference is that we have to accomodate the presence of an ontology. The first step of this proof is to
provide a model-theoretic characterization for interpolant existence in terms of bisimulations, inspired by early works by Robinson~\cite{AbRobinson}. This is a rather uniform step that can be adapted to
virtually every (description) logic by ``plugging in'' the respective bisimulation
notion capturing the expressive power of the logic.
The second step is then an amalgamation lemma which 
shows that the model-theoretic condition is satisfied whenever the
subsumption holds. This second step does not work for all logics, and
we will see that failure of this step for some logic usually means that this logic does not enjoy the CIP.

We introduce the necessary notation. Let $\Lmc$ be some description logic. Let $C_1,C_2$ be $\Lmc$ concepts, $\Omc$ be
an \Lmc ontology, and $\Sigma$ be a signature. Then $C_1, C_2$ are called
\emph{jointly $\sim_{\Lmc,\Sigma}$-consistent under \Omc} if
there exist models $\Imc_1,\Imc_2$ of \Omc and elements $d_i\in C_i^{\Imc_i}$ for $i=1,2$,
satisfying
$\Imc_1,d_1\sim_{\Lmc,\Sigma}\Imc_2,d_2$.\footnote{It is worth noting that \emph{joint
$\sim_{\Lmc,\Sigma}$-consistency} is dual to the notion of \emph{entailment along
bisimulations} in~\refchapter{chapter:sixproofs}. Indeed, $C_1$ and $\neg C_2$
are jointly $\sim_{\Lmc,\Sigma}$-consistent iff $C_1$ does not
entail $C_2$ along $\Lmc(\Sigma)$-bisimulations. We stick to
joint bisimilarity since it has been the term recently used in the DL
community.} Joint consistency can be thought of as a strong form of satisfiability
of two concepts with an additional bisimilarity requirement. The following lemma
characterizes the existence of interpolants in terms of joint consistency. The
proof is rather standard and relies on the fact that \Lmc-bisimulations capture the expressive power of \Lmc and crucially also on compactness, see \refchapter{chapter:sixproofs} for a similar lemma, but also~\cite{goranko20075}.
\begin{lemma} \label{lem:joint-consistency}
  Let $C_1,C_2$ be \Lmc concepts, $\Omc$ an $\Lmc$ ontology, and $\Sigma$ be a signature. Then the following conditions are equivalent:
  \begin{enumerate}

    \item there is no $\Lmc(\Sigma)$-interpolant for $C_1$ and $C_2$ under
      $\Omc$;

    \item\label{eq:joint-consistent} $C_1$ and $\neg C_2$ are jointly $\sim_{\Lmc,\Sigma}$-consistent under
      $\Omc$. 

  \end{enumerate}
\end{lemma}
Intuitively, for an $\Lmc(\Sigma)$-interpolant for $C_1$ and $C_2$ to exist, the inconsistency of $C_1$ and $\neg C_2$ must be detectable by an $\Lmc(\Sigma)$-bisimulation, since otherwise we cannot express it using an $\Lmc(\Sigma)$ concept.
%
Hence, in order to show the CIP for a DL, it suffices to show that a witness for Point~\ref{eq:joint-consistent} can be turned into a witness for failure of the subsumption $\Omc\models C_1\sqsubseteq C_2$. This is the content of the following amalgamation lemma, which is stated and proved for \ALC only, for the sake of simplicity. Its proof is similar to the proof of an analogous amalgamation lemma for modal logic in \refchapter{chapter:sixproofs}. 
We refer the reader to \refchapter{chapter:algebra} for more on amalgamation.
\begin{lemma}\label{lem:amalgamation}
  Let $\Imc_1,\Imc_2$ be models of an \ALC ontology $\Omc_1\cup\Omc_2$ and suppose $\Imc_1,d_1\sim_{\ALC,\Sigma_1\cap \Sigma_2}\Imc_2,d_2$ for $d_1\in\Delta^{\Imc_1},d_2\in\Delta^{\Imc_2}$ and signatures $\Sigma_1\supseteq \sig{\Omc_1},\Sigma_2\supseteq \sig{\Omc_2}$. Then there is a model $\Jmc$ of $\Omc_1\cup\Omc_2$ and $e\in\Delta^{\Jmc}$ such that $\Jmc,e\sim_{\ALC,\Sigma_1}\Imc_1,d_1$ and $\Jmc,e\sim_{\ALC,\Sigma_2}\Imc_2,d_2$.
\end{lemma}

As announced, Lemmas~\ref{lem:joint-consistency} and~\ref{lem:amalgamation} can be used to prove that \ALC enjoys the CIP.  
\begin{proof}[Proof of Theorem~\ref{thm:cip} for \ALC]
Suppose $\Omc_1\cup\Omc_2\models C_1\sqsubseteq C_2$ for \ALC ontologies
$\Omc_1,\Omc_2$ and \ALC concepts $C_1,C_2$ with
$\Sigma_1=\sig{\Omc_1,C_1}$ and $\Sigma_2=\sig{\Omc_2,C_2}$. If there is no Craig
interpolant for $C_1$ and $C_2$ under $\Omc_1\cup\Omc_2$, then by
Lemma~\ref{lem:joint-consistency}, there are models $\Imc_1,\Imc_2$ of $\Omc_1\cup\Omc_2$ and $d_1\in\Delta^{\Imc_1},d_2\in\Delta^{\Imc_2}$ such that
$d_1\in C_1^{\Imc_1}$, $d_2\notin C_2^{\Imc_2}$, and
$\Imc_1,d_1\sim_{\ALC,\Sigma_1\cap \Sigma_2} \Imc_2,d_2$. By Lemma~\ref{lem:amalgamation}, there is a model $\Jmc$ of $\Omc_1\cup\Omc_2$ and $e\in \Delta^\Jmc$ such that $\Jmc,e\sim_{\ALC,\Sigma_1}\Imc_1,d_1$ and $\Jmc,e\sim_{\ALC,\Sigma_2}\Imc_2,d_2$. Now, since $d_1\in C_1^{\Imc_1}$, $\sig{C_1}\subseteq \Sigma_1$, and $\Jmc,e\sim_{\ALC,\Sigma_1}\Imc_1,d_1$, we also have
$e\in C_1^{\Jmc}$. Analogously, one can show that $d_2\notin C_2^{\Jmc}$. This is in contradiction to the assumed subsumption $\Omc_1\cup\Omc_2\models C_1\sqsubseteq C_2$.
\end{proof}

We come back to the problem of \Lmc-interpolant existence. As mentioned above, the problem trivializes for DLs that enjoy the CIP if we ask for interpolants in the common signature.
Moreover, it is not difficult to reduce the existence of $\Lmc(\Sigma)$-interpolants for given signature $\Sigma$ to the existence of Craig interpolants. Indeed, one can verify that, for any standard DL \Lmc, the following are equivalent for all \Lmc concepts $C_1,C_2$, \Lmc ontologies $\Omc$, signatures $\Sigma$, and $C_{2\Sigma}$ and $\Omc_\Sigma$ obtained from $C_2$ and $\Omc$, respectively, by renaming all symbols not in $\Sigma$ uniformly to fresh symbols:
\begin{itemize}

  \item there is an $\Lmc(\Sigma)$-interpolant for $C_1$ and $C_2$ under $\Omc$;
  
  \item there is a Craig interpolant for $C_1$ and $C_{2\Sigma}$ under $\Omc\cup \Omc_\Sigma$.

\end{itemize}
Since existence of Craig interpolants coincides with the respective subsumption relationship in DLs enjoying the CIP, we have reduced \Lmc-interpolant existence to subsumption checking for such DLs. A reduction of subsumption to interpolant existence is also possible. Since subsumption in all DLs mentioned in Theorem~\ref{thm:cip} is \ExpTime-complete~\cite{DBLP:phd/dnb/Tobies01}, we obtain:   
\begin{theorem}\label{thm:interpolant-existence-cip}
  For any DL \Lmc in Theorem~\ref{thm:cip}, \Lmc-interpolant existence is \ExpTime-complete. 
\end{theorem}

It is worth noting that the proof of Theorem~\ref{thm:cip} (and thus of Theorem~\ref{thm:interpolant-existence-cip}) is not constructive since the proof of the model-theoretic characterization in Lemma~\ref{lem:joint-consistency} relies on compactness and does not provide an interpolant if one exists. We will address the computation problem in Section~\ref{sec:interpolant-computation}. 

Here, we will turn our attention to the fact that, unfortunately, not all DLs enjoy the CIP.
\begin{theorem}[\cite{TenEtAl13,DBLP:journals/tocl/ArtaleJMOW23}]\label{thm:alcocip}
  \ALCO and \ALCH and their extensions by $\Smc,\Imc,\Fmc$ do not enjoy the Craig interpolation property.
\end{theorem}
\begin{proof}
  We show the argument only for \ALCO, for \ALCH and their extensions see~\cite{TenEtAl13,DBLP:journals/tocl/ArtaleJMOW23}.
  Consider \ALCO concepts $C_1=\{a\}\sqcap
  \exists r.\{a\}$ and $C_2=A\to\exists r.A$ with common signature $\{r\}$. Then $\models
  C_1\sqsubseteq C_2$, but there is no $\ALCO(\{r\})$-interpolant for
  $C_1$ and $C_2$.
  Indeed, the interpretations in
  Figure~\ref{fig:alcocip} witness that $C_1$ and $\neg C_2$ are jointly
  $\sim_{\ALCO,\{r\}}$-consistent, and by
  Lemma~\ref{lem:joint-consistency} there is no Craig interpolant.
\end{proof}

\begin{figure}[t]
  \centering
  \begin{tikzpicture}
    \tikzset{
      dot/.style = {draw, fill=black, circle, inner sep=0pt, outer sep=1pt, minimum size=2pt}
    }

    \draw (-2,0.5) node[label=$\Imc_{1}$] (I1) {};

    \draw (0,0) node[dot, label={[shift={(-0.25,-0.25)},align=center]:$a$}, label={[shift={(0,-0.75)},align=center]:$C_{1}$}] (a) {};

    \draw (-1,0) node[dot, label={[shift={(-0.25,-0.25)},align=center]:$b$}] (b1) {};


    \draw[->, >=stealth]  (a) edge [loop above] node[] {$r$} ();

    \draw (7,0.5) node[label=$\Imc_{2}$] (I2) {};

    \draw (6,-0) node[dot, label={[shift={(0.25,-0.25)},align=center]:$a$}] (a2) {};

    \draw (5,-0) node[dot, label={[shift={(-0.3,-0.4)},align=center]:$A$}, label={[shift={(0.25,-0.25)},align=center]:$b$}, label={[shift={(0,-0.75)},align=center]:$\lnot C_{2}$}] (b2) {};

    \draw (5,1) node[dot, label=east:$d$] (d) {};


    \draw[->, >=stealth] (b2) -- (d) node[midway, left] {$r$};

    \draw[->, >=stealth]  (d) edge [loop above] node[] {$r$} ();


    \path[black, dashed, bend left] (a) edge (b2);
    \draw (2.5,1.25) node[label=$\sim_{\mathcal{ALCO}, \Sigma}$] () {};

    \path[black, dashed, bend left] (a) edge (d);
    \draw (2.5,0) node[label=$\sim_{\mathcal{ALCO}, \Sigma}$] () {};

  \end{tikzpicture}

  \caption{Interpretations $\Imc_{1}$ and $\Imc_{2}$ illustrating
  the proof of Theorem~\ref{thm:alcocip}.
  }
  \label{fig:alcocip}
\end{figure}

One consequence of the lack of the CIP is that the reduction of
interpolant existence to subsumption checking used to prove Theorem~\ref{thm:interpolant-existence-cip} does not work anymore.
Fortunately, it turns out that the problem is still decidable, although typically
harder than subsumption. We exploit that Lemma~\ref{lem:joint-consistency}
provides a reduction of interpolant existence to the problem of deciding joint
consistency under ontologies, so it suffices to decide the latter. For didactic
purposes and since we need it later, we provide first a relatively
simple algorithm for deciding joint consistency in \ALC, which is inspired by
standard type elimination algorithms for deciding satisfiability in
DLs~\cite{DL-Textbook} and is closely related to the type elimination sequences
in \refchapter{chapter:sixproofs}. 

To describe the algorithm deciding joint consistency, let us fix an \ALC ontology \Omc, \ALC concepts $C_1,C_2$, and a signature $\Sigma$.
Let $\Gamma$ denote the set of all subconcepts that occur in \Omc, $C_1,C_2$, closed under single negation. A \emph{type} is any set $t\subseteq \Gamma$ such that there is a model $\Imc$ of \Omc and an element $d\in \Delta^\Imc$ such that $t=\tp_\Imc(d)$ where, as in Section~\ref{sec:prop-ui},
\[\tp_\Imc(d)=\{C\in\Gamma\mid d\in C^\Imc\}.\]
A \emph{mosaic $m$} is a pair $m=(t_1,t_2)$ of types. Intuitively, a mosaic $(t_1,t_2)$ provides an abstract description of two elements $d_1,d_2$ in two models $\Imc_1,\Imc_2$ of \Omc which have types $t_1,t_2$ and are $\ALC(\Sigma)$-bisimilar. Of course, not all mosaics are realizable in this sense and the goal of the elimination is to identify those that are.

We write $t\leadsto_r t'$ if an element of type $t'$ is a viable $r$-successor of an element of type $t$, that is, $\{C\mid \forall r.C\in t\}\subseteq t'$. 
We use $(t_1,t_2)\leadsto_r(t_1',t_2')$ to abbreviate $t_1\leadsto_r t_1'$, $t_2\leadsto_r t_2'$.
Let \Mmc be a set of mosaics. We call $(t_1,t_2)\in \Mmc$ \emph{bad} if it violates one of the following conditions:
\begin{description}
  \item[(Atomic Consistency)\label{desc:atomic}] for every $A\in\Sigma$, $A\in t_1$ iff $A\in t_{2}$;

  \item[(Existential Saturation)\label{desc:existential}] for every $r\in \Sigma$, every $i=1,2$, and every $\exists r.C\in t_i$, there is $(t_1',t_2')\in \Mmc$ such that $C\in t_i'$ and 
  $(t_1,t_2)\leadsto_r (t_{1}',t_2')$.
\end{description}
Clearly, a mosaic $(t_1,t_2)$ violating atomic consistency cannot be realized as required, due to concept name $A\in \Sigma$. Similarly, a mosaic violating existential saturation cannot be realized since it lacks a viable $r$-successor for one of the $t_i$. We have the following characterization of joint consistency: 
\begin{lemma}\label{lem:joint-consistency-elimination}
  $C_1$ and $C_2$ are jointly $\sim_{\ALC,\Sigma}$-consistent under \Omc iff there is a set $\Mmc^*$ of mosaics that does not contain bad mosaics and some $(t_1,t_2)\in \Mmc^*$ with $C_1\in t_1$ and $C_2\in t_2$.
\end{lemma}
%
%
The existence of an $\Mmc^*$ as in Lemma~\ref{lem:joint-consistency-elimination} can be decided by a
simple mosaic elimination algorithm. The idea is to compute a sequence of mosaics 
\[\Mmc_0,\Mmc_1,\Mmc_2,\ldots\]
where $\Mmc_0$ is the set of all mosaics, and for $i\geq 0$, $\Mmc_{i+1}$ is obtained from $\Mmc_i$ by eliminating all bad mosaics from $\Mmc_i$. Since $\Mmc_0$ is finite, this process reaches a fixpoint after finitely many steps. Then, $C_1$ and $C_2$ are jointly $\sim_{\ALC,\Sigma}$-consistent under \Omc iff the fixpoint $\Mmc^*$ of that process contains $(t_1,t_2)$ as in Lemma~\ref{lem:joint-consistency-elimination}.
The algorithm runs in exponential time in the size of the input, since the number of mosaics is at most exponential in the size of the input, and each mosaic can be checked for badness in time polynomial in the number of mosaics. Thus, the approach to interpolant existence via deciding joint consistency is not worse than the sketched reduction to subsumption checking. We will see in Section~\ref{sec:interpolant-computation} how to extend it to compute interpolants as well.  

This idea can be generalized to decide \Lmc-interpolant existence for DLs lacking the CIP.

\begin{theorem}[\cite{DBLP:journals/tocl/ArtaleJMOW23}]\label{thm:interpolant-existence}
  For $\Lmc\in\{\mathcal{ALCO},\mathcal{ALCH},\mathcal{ALCOI},\mathcal{ALCHI}\}$, \Lmc-interpolant existence is \TwoExpTime-complete.
\end{theorem}
Due to Lemma~\ref{lem:joint-consistency}, it suffices to show the results for joint consistency.
We sketch it here for \ALCO, see~\cite{DBLP:journals/tocl/ArtaleJMOW23} for a full proof and the other DLs covered in the theorem. To get an intuition for the lower bound proof, it is instructive to look again at the proof of Theorem~\ref{thm:alcocip} and particularly Figure~\ref{fig:alcocip}, where the elements $b,d$ in $\Imc_2$ are \emph{forced} to be bisimilar, as they are both bisimilar to $a$ in $\Imc_1$. This idea is
extended to force \emph{exponentially many} elements to be bisimilar in witnesses for joint $\sim_{\Lmc,\Sigma}$-consistency, which in turn is exploited to synchronize configurations of expontially space bounded alternating Turing machines. 
 
For the upper bound we extend the mosaic elimination algorithm that was used to decide joint consistency for \ALC to \ALCO. Let $\Omc$ be an \ALCO ontology, $C_1,C_2$ be \ALCO concepts, and $\Sigma$ be a signature. Let $\Gamma$ be again the set of all subconcepts of $\Omc,C_1,C_2$. A \emph{type} is defined as before as a subset of $\Gamma$ that is realizable in a model of \Omc. 
To address that one can force elements to be bisimilar, we have to extend the notion of a mosaic $m$ to be a pair $(T_1,T_2)$ of \emph{sets} of types. Intuitively, a mosaic $(T_1,T_2)$ describes collections of elements in two interpretations $\Imc_1,\Imc_2$ which realize precisely the types in $T_1,T_2$ and are mutually $\ALCO(\Sigma)$-bisimilar. Naturally, not all such mosaics can be realized as described and the goal is to find the realizable ones. We write $(T_1,T_2)\leadsto_r (T_1',T_2')$ if for $i=1,2$ and every $t\in T_i$, there is $t'\in T_i'$ with $t\leadsto_r t'$, that is, every type in $T_i$ has a viable $r$-successor in $T_i'$.

Let $\Mmc$ be a set of mosaics. A mosaic $(T_1,T_2)\in \Mmc$ is \emph{bad} if it violates one of the conditions atomic consistency and existential saturation, suitably generalized to sets of types:
\begin{description}
  \item[(Atomic Consistency)] for every $A\in\Sigma$ and every $t,t'\in T_1\cup T_2$, $A\in t$ iff $A\in t'$;
  \item[(Existential Saturation)] for every $r\in \Sigma$, $i=1,2$, every $t\in T_i$, and every $\exists r.C\in t$, there is $(T_1',T_2')\in \Mmc$ with $(T_1,T_2)\leadsto_r (T_1',T_2')$ and such that $C\in t'$ for some $t'\in T_i'$ with $t\leadsto_r t'$.

\end{description}
We need a further property of sets of mosaics to ensure that nominals are handled correctly: they are realized precisely once in every interpretation. A set $\Mmc$ of mosaics is \emph{good for nominals} if for every individual $a\in \Gamma$ and $i=1,2$, there is exactly one $t^i_a$ with $\{a\}\in t^i_a\in\bigcup_{(T_1,T_2)\in \Mmc} T_i$ and exactly one pair $(T_1,T_2)\in\Mmc$ with $t_a^i\in T_i$. Moreover, if $a\in \Sigma$, then the pair takes one of the forms $(\{t^1_a\},\{t^2_a\})$, $(\emptyset,\{t^2_a\})$, or $(\{t^1_a\},\emptyset)$.
One can then show:
\begin{lemma}[Lemma~6.5 in~\cite{DBLP:journals/tocl/ArtaleJMOW23}]\label{lem:alco-joint-consistency}
  $C_1$ and $C_2$ are jointly $\sim_{\ALCO,\Sigma}$-consistent iff there is a set $\Mmc^*$ of mosaics that is good for nominals and contains no bad mosaic such that there is $(T_1,T_2)\in\Mmc^*$ and $t_1\in T_1, t_2\in T_2$ with $C_1\in t_1$ and $C_2\in t_2$.
\end{lemma}
The upper bound in Theorem~\ref{thm:interpolant-existence} now follows from the fact that there are only double exponentially many mosaics and that an $\Mmc^*$ as in Lemma~\ref{lem:alco-joint-consistency} can be found (if it exists) by eliminating mosaics from all \emph{maximal} sets of mosaics that are good for nominals.  The upper bound for \ALCH and the other logics in Theorem~\ref{thm:interpolant-existence} is similar. In fact, mosaic elimination has been used to show decidability of interpolant existence also for other logics that related to DLs and relevant to KR~\cite{DBLP:conf/lics/JungW21,DBLP:conf/kr/KuruczWZ23,DBLP:conf/aiml/KuruczWZ24}.

\smallskip\textbf{\EL and its Relatives.} 
Given the importance of the lightweight DL \EL in applications, we report also on the (surprisingly different) situation in \EL and its relatives. While \EL itself and its extension with role hierarchies or transitive roles do enjoy the CIP~\cite{DBLP:journals/lmcs/LutzSW19,KR2022-16}, most other extensions of \EL such as with inverse roles, nominals, or the combination of transitive roles and role hierarchies do not~\cite{KR2022-16}. Thus, the extension with \Hmc does not lead to the loss of the CIP in \EL, but it does in \ALC, and conversely, the extension with \Imc does lead to the the loss of the CIP in \EL, but not in \ALC. As a rule of thumb, the complexity of the interpolant existence problem in an extension of \EL without CIP coincides with the complexity of subsumption in that extension~\cite{KR2022-16}. This is in contrast to \ALC where interpolant existence is generally one exponent more difficult than subsumption for extensions without CIP.

\smallskip \textbf{Repairing the CIP.} We conclude this part with a discussion of \emph{repairing} the lack of the CIP by allowing interpolants from a richer logic. A notable positive instance of this approach is to allow formulas in the guarded, two-variable fragment (GF$^2$) of FO as interpolants under \ALCH ontologies: GF$^2$ extends \ALCH and it is known that GF$^2$ does enjoy the CIP~\cite{DBLP:journals/sLogica/HooglandM02,DBLP:conf/lpar/HooglandMO99}. Consequently, there is a GF$^2$-interpolant for every valid subsumption in \ALCH. The situation is more complicated in the case of \ALCO. While it would be interesting to investigate existence of interpolants in logics more expressive than \ALCO, say GF$^2$ with constants, it has been shown that (under mild conditions) there is no \emph{decidable} extension of \ALCO that enjoys CIP~\cite{DBLP:journals/jsyml/Cate05}. Hence, there is no decidable logic which guarantees the existence of interpolants for every valid subsumption in \ALCO.



\subsection{Beth Definability}\label{sec:beth-definability}

Since interpolation is intimately tied to definability, c.f.\ \refchapter{chapter:firstorder} and  \refchapter{chapter:queryrewriting}, and since definability is a central topic when working with DL ontologies,
in
this section we take a closer look at this connection in the context of DLs.
We
start with introducing the relevant notions. Recall that, for a signature $\Sigma$, the \emph{$\Sigma$-reduct $\Imc|_\Sigma$} of an interpretation $\Imc$ is the interpretation obtained from $\Imc$ by dropping the interpretation of all symbols not in $\Sigma$.
%
\begin{definition}[Explicit/Implicit Definability]
 Let \Lmc be any DL. Let $\Omc$ be an \Lmc ontology, $C,C_{0}$ be \Lmc concepts, and $\Sigma$
be a signature.
We say that $C_0$ is:
\begin{itemize}
\item \emph{explicitly
$\Lmc(\Sigma)$-definable
under $\Omc$ and $C$} if there is an explicit $\Lmc(\Sigma)$-definition of $C_{0}$
under $\Omc$ and $C$, that is, an $\Lmc(\Sigma)$ concept $D$ satisfying $\Omc \models C \sqsubseteq (C_{0}\leftrightarrow D)$, and
\item \emph{implicitly $\Sigma$-definable under $\Omc$ and $C$} if for all models $\Imc$ and $\Jmc$ of $\Omc$ with $\Imc_{|\Sigma} = \Jmc_{|\Sigma}$ and all $d\in C^\Imc$, we have $d\in C_{0}^{\Imc}$ iff $d\in
  C_{0}^{\Jmc}$.
\end{itemize}
\end{definition}
Intuitively, $C_{0}$ is implicitly $\Sigma$-definable under $\Omc$ and $C$ if for every model $\Imc$ of \Omc and all $d\in C^\Imc$,
the $\Sigma$-reduct $\Imc|_\Sigma$ ``determines'' whether $d\in C_0^\Imc$.
Implicit definability can be equivalently defined in terms of subsumption: $C_{0}$ is implicitly
$\Sigma$-definable under $\Omc$ and $C$ iff 
\begin{eqnarray}\label{eq:impl} 
  \Omc \cup \Omc_{\Sigma} \models C \sqcap C_{0} \sqsubseteq C_{\Sigma} \rightarrow {C_{0}}_{\Sigma} 
\end{eqnarray} 
where $\Omc_{\Sigma}$, $C_{\Sigma}$, and ${C_{0}}_{\Sigma}$ are obtained from $\Omc$, $C$ and $C_{0}$, respectively, by replacing every non-$\Sigma$ symbol uniformly by a fresh symbol. It is worth noting that a common alternative definition of explicit/implicit definability does not use the context concept $C$; it is a special case of the above definition in which $C$ is set to $\top$. We use the more general definition in order to establish a stronger relation to Craig interpolation later on. The following example illustrates implicit and explicit definability and makes do with the simpler definition.
\begin{example}
    Consider the \ALC ontology consisting of the following CIs:
    \begin{align*}
      \textsf{Parent}& \equiv \exists\textsf{child}.\top & \textsf{Parent}& \equiv \textsf{Father}\sqcup \textsf{Mother} \\
      \textsf{Father}& \sqsubseteq \textsf{Man} & \textsf{Mother}& \sqsubseteq \textsf{Woman}
      & \textsf{Man} & \sqsubseteq \neg \textsf{Woman} 
    \end{align*}
    Then, $\textsf{Mother}$ is implicitly $\Sigma$-definable under
    \Omc and $C=\top$, for $\Sigma=\{\textsf{hasChild},\textsf{Woman}\}$. Indeed, in any model $\Imc$ of \Omc, any element that satisfies \textsf{Woman} and has an $\textsf{hasChild}$-successor has to satisfy $\textsf{Mother}$, and other elements can not satisfy $\textsf{Mother}$.  
    This is equivalent with saying that $\textsf{Woman}\sqcap \exists \textsf{hasChild}.\top$ is an explicit $\ALC(\Sigma)$-definition of $\textsf{Mother}$ under $\Omc$.
   \lipicsEnd
  \end{example}

  It should be clear that, if a concept is explicitly
  $\Lmc(\Sigma)$-definable
  under $\Omc$ and $C$, then it is implicitly $\Sigma$-definable
  under $\Omc$ and $C$, for any language $\Lmc$. A logic enjoys the
  projective Beth definability property if the converse implication
  holds as well.
  \begin{definition}\label{def:pbdp}
    A DL $\Lmc$ \emph{enjoys the projective Beth definability
    property} (\emph{PBDP}) if for any $\Lmc$ ontology $\Omc$,
    $\Lmc$ concepts $C$ and $C_{0}$, and signature $\Sigma\subseteq
    \sig{C,\Omc}$ the following holds: if $C_{0}$ is implicitly
    $\Sigma$-definable under $\Omc$ and $C$, then $C_{0}$ is
    explicitly $\Lmc(\Sigma)$-definable under $\Omc$ and $C$.
  \end{definition}
Explicit definability is, similarly to interpolant existence, characterized via joint consistency.
\begin{lemma}\label{lem:pdbp-joint}
  Let $\Lmc$ be any DL. Let
  $C_0,C$ be \Lmc concepts, $\Omc$ an $\Lmc$ ontology, and $\Sigma$ be a signature. Then the following conditions are equivalent:
  \begin{enumerate}

    \item there is no explicit $\Lmc(\Sigma)$-definition for $C_0$ under \Omc and $C$;

    \item $C\sqcap C_0$ and $C\sqcap \neg C_0$ are jointly $\sim_{\Lmc,\Sigma}$-consistent under
      $\Omc$.

  \end{enumerate}
\end{lemma}
Based on Lemma~\ref{lem:pdbp-joint}, one can prove the equivalence of PBDP and the CIP.   
%
\begin{lemma}\label{lem:CIPBDP}
  For every extension $\Lmc$ of \ALC with any of \Smc,\Imc,\Fmc,\Omc,\Hmc, \Lmc enjoys the CIP iff \Lmc enjoys the PBDP. Moreover, there are polynomial time reductions between computing interpolants and computing explicit definitions.
\end{lemma}

\begin{proof}
  The proof of ``$\Rightarrow$'' is by a standard argument~\cite{GabbayMaks05}:
  Assume that an $\Lmc$ concept $C_{0}$ is implicitly
  $\Sigma$-definable under an $\Lmc$ ontology $\Omc$ and
  $\Lmc$ concept $C$, for some signature $\Sigma$.  Then
  \eqref{eq:impl} holds. Take an $\Lmc(\Sigma)$-interpolant $I$ for $C
  \sqcap C_{0}$ and $C_{\Sigma} \rightarrow {C_{0}}_{\Sigma}$
  under $\Omc\cup \Omc_{\Sigma}$.  Then $I$ is an explicit
  $\Lmc(\Sigma)$-definition of $C_{0}$ under $\Omc$ and~$C$.

  For ``$\Leftarrow$'', suppose $\Omc_1\cup\Omc_2\models C\sqsubseteq D$, for \Lmc ontologies $\Omc_1,\Omc_2$ and \Lmc concepts $C,D$, and let $\Sigma = \sig{\Omc_1,C_1}\cap\sig{\Omc_2,C_2}$. Based on Lemmas~\ref{lem:joint-consistency} and~\ref{lem:pdbp-joint}, one can show that there is a Craig interpolant for $C$ and $D$ under $\Omc_1\cup\Omc_2$ iff there is an explicit $\Lmc(\Sigma)$-definition of $D$ under $D\to C$ and $\Omc_1\cup\Omc_2$. Finally observe that $D$ is implicitly $\Sigma$-definable under $\Omc_1\cup\Omc_2$ and $D\to C$: The right-hand side of the concept inclusion in~\eqref{eq:impl} is $(D_\Sigma\to C_\Sigma)\to D_\Sigma$, a tautology.
\end{proof}
%
Lemma~\ref{lem:CIPBDP} is remarkable since for many logics the CIP is strictly stronger than the PBDP, see \refchapter{chapter:nonclassical} and for example~\cite{DBLP:journals/apal/Maksimova00}.
On the one hand, the proof of ``$\Rightarrow$'' of Lemma~\ref{lem:CIPBDP} is rather robust and also applies to weaker logics, such as \EL, and weaker versions of explicit/implicit definability, such as the one with $C=\top$. On the other hand, the proof of ``$\Leftarrow$'' relies on both the presence of the context concept $C$ and on the expressive power of \ALC. In fact, we conjecture that the ``$\Leftarrow$''-direction does not hold for the weaker definition with $C=\top$. 


Taking into account our knowledge about the CIP from Theorems~\ref{thm:cip} and~\ref{thm:alcocip}, we obtain:
%
\begin{corollary}\label{cor:pbdp}
  \ALC and its extensions by \Smc, \Imc, \Fmc enjoy the PBDP, but \ALCO and \ALCH and their extensions by \Smc, \Imc, \Fmc do not.
\end{corollary}

We conclude the section with a brief discussion of two weaker forms of Beth definability that have also been considered in the DL literature.
\begin{itemize}

\item In the \emph{non-projective Beth definability property (BDP)}, we are only interested in defining concept names $C_0=A$ in terms of all other symbols in the signature, that is, $\Sigma=\sig{\Omc}\setminus\{A\}$. Clearly, the PBDP implies the BDP, but the other direction is not true in general. For example, \ALCH enjoys the BDP, but it lacks the PBDP. \ALCO does not enjoy already BDP. We refer the reader to~\cite{DBLP:journals/tocl/ArtaleJMOW23} for a broader discussion.

\item In the \emph{concept-based Beth definability property (CBDP)} both implicit and explicit definability are defined in terms of signatures which always contain all relevant role names, that is, only concept names may be restricted. In the same way, one can define a corresponding \emph{concept Craig interpolation property (CCIP)}, and it has been shown recently using a sequent calculus that the highly expressive DL $\mathcal{RIQ}$ enjoys both CBDP and CCIP~\cite{DBLP:conf/ijcai/LyonK24}. 

\end{itemize}

\subsection{Applications}\label{sec:applications}

We discuss several applications of interpolants and explicit definitions, starting
with a recently discovered relation of Craig interpolants to  \emph{description logic concept learning} in the presence of nominals. Informally, concept learning is the task of inducing a concept description from sets of positive
and negative data examples, which has received a lot of interest in recent years, see e.g.~\cite{DBLP:journals/ml/LehmannH10,DBLP:journals/ai/JungLPW22,DBLP:conf/ijcai/FunkJLPW19}. The predominant application scenario is reverse-engineering of concepts to support ontology engineers in writing complex concepts. To make the connection to interpolation precise, we introduce some notation.

In our context, an \emph{example} is a pair $(\Dmc,a_0)$, where $\Dmc$ is a database, that is, a finite set of facts of the form $A(a)$ and $r(a,b)$ with $A$ a concept name, $r$ a role name, and $a,b$ individuals, and $a_0$ is an individual from database \Dmc. 
A \emph{labeled data set} is a pair $(P,N)$ of sets of positive and negative examples. 
Let \Lmc be some DL, $\Omc$ be an \Lmc ontology, $\Sigma$ a signature, and $(P,N)$ a labeled dataset. An \emph{$\Lmc(\Sigma)$-separator} for
$\Omc,P,N$ is an $\Lmc(\Sigma)$ concept $C$ such that
\begin{itemize}
\item $\Omc\cup \Dmc\models C(a)$ for all $(\Dmc,a)\in P$ and
\item $\Omc\cup \Dmc\models \neg C(a)$ for all $(\Dmc,a)\in
N$.
\end{itemize}
Intuitively, a separator is a potentially complex concept that distinguishes between the positive and negative examples in the sense of consistent learning. 
We refer to the induced problem of deciding the existence of an $\Lmc(\Sigma)$-separator for given $\Omc,\Sigma,P,N$ with \emph{\Lmc-separator existence}. 

It has been recently observed that there is a strong connection between interpolants and separators in many DLs extending \ALCO. More precisely, for such DLs \Lmc there are polynomial time reductions between \Lmc-interpolant existence and \Lmc-separator existence. Moreover, the reductions guarantee that interpolants can be used as separators and vice versa~\cite[Theorem~4.3]{DBLP:journals/tocl/ArtaleJMOW23}. The presence of nominals is crucial in these reductions as they are needed to encode the database. Theorem~\ref{thm:interpolant-existence} implies that separator existence in \ALCO and its extenions is a hard yet decidable problem.

A special case of separator existence is the existence of \emph{referring expressions}.
Referring expressions originate in linguistics to refer to a single object in the domain of discourse~\cite{Cann_1993}, but have recently been introduced in data management and KR~\cite{DBLP:conf/inlg/ArecesKS08,DBLP:journals/coling/KrahmerD12,DBLP:conf/kr/BorgidaTW16,ArtEtAl21}. Existence of a referring expression for object $a$ can be cast as a separator existence problem, 
in which the task is to find a separator between $a$ and all other individuals in the domain of discourse. Indeed, in many cases, the concrete individual, say the internal id $p12345$, could be less meaningful to the user than a concise description in terms of a given signature, e.g., \emph{head of finance}. Referring expressions are also a special case of explicit definitions. 

Consider next the process of \emph{ontology construction}. It offers essentially two ways to extend a present ontology by a new concept name, say $A$~\cite{TenEtAl06}. First, in the \emph{explicit} manner, one adds a concept definition $A\equiv C$ for some concept description $C$ over the signature of the present ontology. Second, in the \emph{implicit} manner, one can add (several) concept inclusions with the property that in every model of the ontology, the interpretation of $A$ is uniquely determined by the interpretation of the other symbols---in other words: $A$ is implicitly defined. Of course, the explicit way is more transparent and thus to be preferred. However, if the ontology language of choice enjoys the PBDP, there is actually an explicit definition for $A$, also in the second case. The notion of PBDP has been studied in this sense under the disguise of \emph{definitorial completeness} in~\cite{DBLP:conf/dlog/BaaderN03,TenEtAl06}.

Finally, we discuss ontology-mediated query answering (OMQ). Recall that in OMQ we access a knowledge base consisting of ontology \Omc and database \Dmc using a query $\varphi(\vec x)$ with the goal to determine the \emph{certain answers}, that is, all tuples $\vec a$ such that $\Omc\cup\Dmc\models \varphi(\vec a)$. Depending on the ontology language, this is a potentially difficult entailment problem and the question is whether one can \emph{rewrite} $\varphi$ and $\Omc$ into $\varphi_\Omc$ such that $\Omc\cup\Dmc\models \varphi(\vec a)$ iff $\Dmc\models \varphi_\Omc(\vec a)$, for all $\vec a$. Note that in this case $\varphi_\Omc$ can be answered efficiently by a standard database system. The PBDP and the connection to interpolation from Lemma~\ref{lem:CIPBDP} have been used several times to study the existence and computation of such rewritings~\cite{DBLP:conf/ijcai/SeylanFB09,DBLP:conf/birthday/FranconiK19,DBLP:journals/jair/FranconiKN13,DBLP:conf/aaai/TomanW22}. \refchapter{chapter:queryrewriting} discusses further applications of the PBDP in the context of database query answering.



\subsection{Computing Interpolants}\label{sec:interpolant-computation}

In the previous sections, 
we have motivated the need for computing interpolants and explicit definitions in various applications. Unfortunately, in spite of that, we are not aware of any DL reasoner that supports the actual computation of interpolants. Hence, this section is restricted to the known theoretical results. By Lemma~\ref{lem:CIPBDP}, we can focus on computing interpolants. 
We concentrate on extensions of \ALC that enjoy the CIP and refer the reader to~\cite{KR2022-16} for \EL and its relatives. In the following theorem from~\cite{TenEtAl13}, we refer to a DAG-representation in which each subconcept is represented only once as \emph{succinct representation} of concepts.
\begin{theorem}[Theorems~3.10, 3.26 in~\cite{TenEtAl13}]\label{thm:interpolant-construction}
Let \Lmc be \ALC or any extension with \Smc, \Imc, \Fmc. Then, one can compute, given \Lmc ontologies $\Omc_1,\Omc_2$ and \Lmc concepts $C_1,C_2$ with $\Omc_1\cup\Omc_2\models C_1\sqsubseteq C_2$, in double exponential time a Craig interpolant of $C_1$ and $C_2$ under $\Omc_1\cup\Omc_2$. If concepts are represented succinctly, then the interpolant can be computed in exponential time.
\end{theorem}
In the same paper, it is shown that the statement in Theorem~\ref{thm:interpolant-construction}
is essentially optimal. We give the lower bound in terms of explicit definitions; it transfers to Craig interpolants via (the proof of) Lemma~\ref{lem:CIPBDP}. It is also not difficult to see that it remains valid when we consider extensions of \ALC with \Smc,\Imc,\Fmc.

\begin{theorem}[Theorem~4.11 in~\cite{TenEtAl13}]\label{thm:interpolant-size-lower}
  Let $\Sigma=\{r,s\}$ consist of two role names. For every $n\geq 0$, there are an \ALC concept $C_n$ and an \ALC ontology $\Omc_n$ of size polynomial in $n$ such that $\Sigma\subseteq \sig{\Omc_n,C_n}$ and $C_n$ is implicitly $\Sigma$-definable under $\Omc_n$, but the smallest explicit $\ALC(\Sigma)$-definition of $C_n$ under $\Omc_n$ is double exponentially long in $n$.
\end{theorem}

\begin{proof}[Proof sketch]
  The proof is via a well-known \emph{path-set
  construction}~\cite{DBLP:conf/atal/Lutz06}. The idea is that $C_n$ together with $\Omc_n$ enforces a full binary $r/s$-tree of depth $2^n$. One then shows that any $\ALC(\Sigma)$ concept $D$ that explicitly defines such a tree is double exponentially long in $n$, since each path in the tree has to be reflected in a (different) path in the syntax tree of~$D$.
\end{proof}

The original proof of Theorem~\ref{thm:interpolant-construction} is by an extension of the tableau-based proof of the CIP for the respective DLs in~\cite{TenEtAl13}. Instead of giving it here, we provide two different perspectives on the computation problem, focusing on \ALC
for the sake of simplicity.
The first one is a reduction of computing interpolants under ontologies to the case without ontologies.  The second one is a refined analyis of the mosaic elimination algorithm that decides joint $\sim_{\ALC,\Sigma}$-consistency given in Section~\ref{sec:interpolation}. The benefit of the former is that this enables us to use all methods for computing interpolants in modal logic presented in \refchapter{chapter:sixproofs}. The benefit of the latter is that it provides the same upper bounds as claimed in Theorem~\ref{thm:interpolant-construction}.

\paragraph*{Perspective 1: The Reduction} The main idea is to \emph{materialize} the ontology to sufficient depth. To formalize this, let $\ALC_n(\Sigma)$ denote the set of all $\ALC(\Sigma)$ concepts of role depth at most $n$, and 
define for each \ALC ontology \Omc using role names $\Sigma_R$ and each $n\geq 0$, a concept $\Phi_\Omc^n\in\ALC_n(\Sigma)$ by taking
\[\Phi_\Omc^n = \bigsqcap_{m\leq n}\ \bigsqcap_{r_1,\ldots, r_m\in \Sigma_R}\forall
r_1\ldots\forall r_m.\bigsqcap_{C\sqsubseteq D\in \Omc} (\neg C\sqcup D). \]
Intuitively, the concept $\Phi_\Omc^n$ expresses that the concept inclusions in
$\Omc$ are satisfied along any $\Sigma_R$-path of length at most $n$. The following lemma, whose proof is similar to the proof of Lemma~9 in~\cite{TenEtAl06}, provides the promised reduction. 
\begin{lemma} \label{lem:reduction-ontology-free}
  Let $\Omc$ be an \ALC ontology, $C_1,C_2$ be
  \ALC concepts, and $N>2^{||\Omc||+||C_1||+||C_2||}$.
  Then, for all signatures $\Sigma$, all $n\geq 0$, and all $I\in\ALC_n(\Sigma)$ the following are equivalent:
\begin{enumerate}

  \item $I$ is an $\ALC(\Sigma)$-interpolant of $C_1$ and $C_2$ under
    $\Omc$;

  \item $I$ is an $\ALC(\Sigma)$-interpolant of
    $\Phi^{N+n}_{\Omc}\sqcap C_1$ and
    $\Phi^{N+n}_{\Omc}\to C_2$.

\end{enumerate}
\end{lemma}
%

Thus, if we were interested in (Craig) interpolants of a certain role
depth $n$, we could use Lemma~\ref{lem:reduction-ontology-free} to compute them
using methods for modal logic from \refchapter{chapter:sixproofs}. Since Theorem~\ref{thm:interpolant-construction} entails an exponential upper bound on the role depth of interpolants (an interpolant that is of exponential size in succinct representation has at most exponential role depth), this is also a complete approach for computing Craig interpolants. We refrain from analyzing the size of interpolants constructed via this method. 

\paragraph*{Perspective 2: Construction from Mosaic Elimination} 
We conduct a refined analysis of the mosaic elimination algorithm that decides joint consistency in \ALC in Section~\ref{sec:interpolation}. 
Let $\Omc,C_1,C_2,\Sigma$ be an input to mosaic elimination.
Recall that the elimination algorithm computes a sequence $\Mmc_0,\Mmc_1,\ldots$ of sets of mosaics, where $\Mmc_0$ is the set of all mosaics, and $\Mmc_{i+1}$ is obtained from $\Mmc_i$ by eliminating bad mosaics. We say that a mosaic $(t_1,t_2)$ is \emph{eliminated in round $\ell$} if $(t_1,t_2)\in \Mmc_\ell$ but $(t_1,t_2)\notin \Mmc_{\ell+1}$. In a slight abuse of notation, we treat a type $t$ as the concept $\bigsqcap_{D\in t}D$.
\begin{lemma}\label{lem:construction-from-elimination}
  If a mosaic $(t_1,t_2)$ is eliminated in round $\ell$, then there is an $\ALC_\ell(\Sigma)$ concept $I$ such that $\Omc\models t_1\sqsubseteq I$ and $\Omc\models t_2\sqsubseteq \neg I$. 
\end{lemma}
\begin{proof}
  The proof is by induction on $\ell$.
  For the base case, let $\ell=0$. If $(t_1,t_2)$ is eliminated in round~0, then this is due to the violation of atomic consistency. Then we can choose $I=A$ or $I=\neg A$ depending on whether the problematic concept name $A$ is contained in $t_1$ or in $t_2$.

  For the inductive case, let $\ell>0$. Then $(t_1,t_2)$ has been eliminated due to violation of existential saturation, that is, there is $r\in\Sigma$, $i\in\{1,2\}$, and $\exists r.C\in t_i$ such that each $(t_1',t_2')\in \Mmc_i$ with $C\in t_i$ satsfies $(t_1,t_2)\not\leadsto_r(t_1',t_2')$. Assume first that $i=1$, and let $T_1,T_2$ be the following sets of types:  
%
\begin{align*}
  T_1 = \{t\mid C\in t\text{ and }t_1\leadsto_r t\}\quad \text{ and }\quad
  T_2 = \{t\mid t_2\leadsto_r t\}.
\end{align*}
%
By induction hypothesis, for each $(t_1',t_2')\in T_1\times T_2$ there is an $\ALC_{\ell-1}(\Sigma)$ concept $I_{t_1',t_2'}$ such that $\Omc\models t_1'\sqsubseteq I_{t_1',t_2'}$ and $\Omc\models t_2'\sqsubseteq \neg I_{t_1',t_2'}$. We choose 
\[I=\exists r.\bigsqcup_{t_1'\in T_1}\bigsqcap_{t_2'\in T_2}I_{t_1',t_2'}.\]
Clearly, $I\in \ALC_\ell(\Sigma)$. We verify $\Omc\models t_1\sqsubseteq I$ and $\Omc\models t_2\sqsubseteq \neg I$.
\begin{itemize}

  \item For the former, let $\Imc$ be a model of $\Omc$ and $d\in t_1^{\Imc}$. Since $\exists r.C\in t_1$, there is some $e\in C^\Imc$ with $(d,e)\in r^\Imc$. Let $t_1'=\tp_\Imc(e)$ and note that $t_1'\in T_1$. By what was said above, we have $\Omc\models t_1'\sqsubseteq I_{t_1',t_2'}$ for all $t_2'\in T_2$. Thus, $d\in (\exists r.I_{t_1',t_2'})^\Imc$ for all $t_2'\in T_2$, and hence $d\in I^\Imc$.

  \item For the latter, let $\Imc$ be a model of $\Omc$ and $d\in t_2^\Imc$. We have to show that $d\notin I^\Imc$, that is, $d\in (\forall r.\bigsqcap_{t_1'\in T_1}\bigsqcup_{t_2'\in T_2}\neg I_{t_1',t_2'})^\Imc$. Suppose $e$ is an $r$-successor of $d$ and let $t_2'= \tp_\Imc(e)$. Note that $t_2'\in T_2$. By what was said above, we have $e\in (\neg I_{t_1',t_2'})^\Imc$ for all $t_1'\in T_1$. Thus $d\notin I^\Imc$.

\end{itemize}
The case $i=2$ is dual; the interpolant has a leading universal quantifier then.
\end{proof}
Using the same arguments, one can finally show the following: 
\begin{lemma}
  For $i=1,2$, let $T_i$ be the set of all types that contain $C_i$. If the result $\Mmc^*$ of the mosaic elimination does not contain any mosaic from $T_1\times T_2$, then 
  \[I=\bigsqcup_{t_1\in T_1}\bigsqcap_{t_2\in T_2}I_{t_1,t_2}\] 
  is an $\ALC_N(\Sigma)$-interpolant of $C_1$ and $\neg C_2$ under \Omc, for $N\leq 2^{||\Omc||+||C_1||+||C_2||}$.
\end{lemma}
A straightforward analysis shows that, moreover, the size of interpolants $I$ constructed in this way is at most double exponential in the size of the input $\Omc,C_1,C_2,\Sigma$, and at most exponential if we allow for succinct representation. Thus, this approach achieves the upper bounds claimed in Theorem~\ref{thm:interpolant-construction}.

We conjecture that it is rather straightforward to extend this way of constructing interpolants to extensions of \ALC enjoying the CIP. It appears to be much more challenging for the extension with $\Hmc$ and/or $\Omc$, which lead to the loss of the CIP. Some progress has recently been made for the extension with $\Hmc$ in~\cite{DBLP:conf/dlog/JungKW25}, which also reports that an earlier algorithm for computing interpolants under \ALCH ontologies from~\cite{DBLP:journals/tocl/ArtaleJMOW23} is not correct (and not easy to fix).

\section{Interpolation in Logic Programming}\label{sec:asp}

In this section, we briefly discuss interpolation in the context of Logic Programming (LP) \cite{BaralKR10}, where applications of interpolation align with the ones discussed in Section~\ref{sec:intro}.
Here, we focus on logic programs under the stable model semantics \cite{GelfondL88}, as this arguably is the most widely used semantics in LP.
In particular, it is the semantics employed in Answer Set Programming (ASP) \cite{ASPPractice12,Lifschitz19}, a form of declarative programming aimed at solving large combinatorial (commonly NP-hard) search problems. 
In addition, a large part of the research in LP on interpolation and, in particular, on the problem of forgetting 
focusses on ASP, motivated by its applications.\footnote{Stable models and answer sets are commonly used synonymously. Though the former is more associated with semantics and the latter with problem solving, the latter has gained wider adoption \cite{lifschitz2008what,brewka2011answer}.}

\subsection{Logic Programs and Answer Sets}\label{sec:programs}

We start by recalling necessary notions and notation for logic programs \cite{ASPPractice12,Lifschitz19}.

We consider first-order signatures $\sign$ consisting of mutually disjoint and countably infinite sets of constants, functions, predicates, and variables.
This allows us to define \emph{terms} as usual as constants, variables, and complex terms $f(t_1,\ldots,t_n)$, where $f$ is a function and $t_1,\ldots t_n$ are terms, as well as \emph{atoms} $p(t_1,\ldots,t_n)$ where $p$ is a predicate and $t_1,\ldots t_n$ are terms.
We denote the set of all atoms over $\Sigma$ with $A(\Sigma)$.

Given $\sign$, a \emph{(logic) program} $P$ is a finite set of \emph{rules} $r$ of the form\footnote{We admit double negation in rules which is historically less common, but needed in our context.}
\begin{align}
	a_1 \vee \ldots \vee a_k  \la b_1,..., b_l, \nf c_{1},..., \nf c_m, \nf \nf d_1,..., \nf \nf d_n  \label{l:rule}
\end{align}
where $\{a_1,\ldots,a_k,b_1,\ldots, b_l,c_{1},\ldots, c_m, d_{1},\ldots, d_n\}\subseteq A(\sign)$.
We also represent such rules with 
\begin{equation}
	\head{r} \la \pbody{r}, \nf \nbody{r}, \nf \nf \nnbody{r} \; , \label{l:shortRule}
\end{equation}
where $\head{r} = \{a_1,\ldots,a_k\}$, $\pbody{r}=\{b_1,\ldots, b_l\}$, $\nbody{r}=\{c_{1},\ldots, c_m\}$, and $\nnbody{r}=\{d_{1},\ldots, d_n\}$.
We distinguish the \emph{head} of $r$, $\rhead{r}=\head{r}$, and its \emph{body}, $\rbody{r}=\pbody{r}\cup\nf \nbody{r}\cup \nf \nf \nnbody{r}$, where 
$\nf \nbody{r}$ and $\nf\nf\nnbody{r}$ represent $\{\nf p\mid p\in \nbody{r}\}$ and $\{\nf\nf p\mid p\in \nnbody{r}\}$, respectively.
We refer to the set of all (logic) programs as the \emph{class of (logic) programs}. 
Different more specific kinds of rules exist: if $n=0$, then $r$ is \emph{disjunctive}; if also $k\leq 1$, then $r$ is \emph{normal}; if in addition $m=0$, then $r$ is \emph{Horn}, if, moreover, $l=0$, then $r$ is a \emph{fact}. 
The class of \emph{disjunctive programs} is defined as a set of finite sets of disjunctive rules, and other classes can be defined accordingly.

For the semantics, we focus on programs in which variables have been instantiated. 
Given a program $P$, this is captured by a ground program, denoted $ground(P)$, which is obtained from $P$ by replacing the variables in $P$ with ground terms over $\Sigma$ in all possible ways.
We also call variable-free terms, atoms, rules, and programs \emph{ground}.
Accordingly, we consider the set of all ground atoms over $\Sigma$ with $A_{gr}(\Sigma)$.
Note that ground programs correspond to propositional programs, where the signature is restricted to a countably infinite set of propositional variables, which can be viewed as atoms built over predicates with arity 0, i.e., no terms, in which case $\Sigma$ would simplify to such a set of nullary predicates.

We define answer sets \cite{GL91}, building on HT-models from the logic of here-and-there \cite{Heyting30}, the monotonic logic underlying answer sets \cite{Pearce99}. 
First, we recall the \emph{reduct} $P^I$ of a ground program $P$ with respect to a set $I$ of ground atoms, adapted here to extended rules:
\[
P^I = \{\head{r}\la \pbody{r} \mid r \text{ of the form (\ref{l:shortRule}) } \text{ in } P \text{ such that } \nbody{r}\cap I=\emptyset \text{ and } \nnbody{r}\subseteq I\}.
\]
The idea is, assuming the atoms in $I$ as true, to transform (ground) $P$ into a corresponding program $P^I$ that does not contain negation.
Then, HT-models are defined as a pair of sets of (ground) atoms, making use of a standard entailment relation $\models$ for programs, where (ground) programs are viewed as a conjunction of implications corresponding to the rules of the form (\ref{l:rule}) with $\neg$, $\wedge$, and $\vee$ denoting classical negation, conjunction, and disjunction, respectively.
Intuitively, in an HT-interpretation $\langle X,Y\rangle$, $Y$ is used to assess satisfaction of $P$ as such, and $X$ is employed to verify satisfaction of the reduct $P^Y$.
An answer set is then a minimal model that satisfies the reduct with respect to itself.

Formally, let $P$ be a ground program. 
An \emph{HT-interpretation} is a pair $\langle X,Y\rangle$ such that $X\subseteq Y \subseteq A_{gr}(\sign)$, and
$\langle X,Y\rangle$ is an \emph{HT-model of $P$} if $Y\models P$ and $X\models P^{Y}$. 
The set of \emph{all HT-models of $P$} is written $\HT(P)$, and $P$ is \emph{satisfiable} if $\HT(P)\not=\emptyset$.
A set of (ground) atoms $Y$ is an \emph{answer set} of $P$ if $\mtuple{Y,Y}\in\HT(P)$, and there is no $X\subsetneq Y$ such that $\mtuple{X,Y}\in\HT(P)$.
We denote by $\as{P}$ the set of all answer sets of $P$, and call $P$ \emph{coherent} if $\as{P}\not=\emptyset$.

We use two entailment relations between logic programs $P_1$, $P_2$ for the two introduced model notions, allowing to specify that $P_1$ entails $P_2$: 
namely, $\htmodels$ represents entailment for HT logic, and $\nccw$ represents 
cautious entailment for answer sets. 
Formally, given programs $P_1$ and $P_2$, $P_1\htmodels P_2$ holds if $\HT(P_1)\subseteq \HT(P_2)$, 
and $P_1\nccw P_2$ holds if $\as{P_1}\subseteq \as{P_2}$.
Note that this also allows us to express that program $P_1$ entails a fact, a set of facts, or a disjunction, etc, as $P_2$ allows to represent these.

\begin{example}\label{exampleAnswerSets}
	Consider for simplicity the following (propositional) program $P$:
	\begin{align*}
		a &\la \nf b &  b & \la \nf c & e & \la d & d &\la a
	\end{align*}
	Then $\langle b,bde\rangle$ is an HT-model of $P$ because $\{b,d,e\}\models P$ and $\{b\}\models P^{\{b,d,e\}}$ where $P^{\{b,d,e\}}=\{b \la  ;\ e  \la d ;\ d \la a\}$.
	Hence, $\{b,d,e\}$ is not an answer set of $P$ as it violates the minimality condition. 
	Moreover, $\{b,d\}$ is not an answer set of $P$ because $\{b,d\}\not\models P$.
	In fact, we can verify that $\{b\}$ is its only answer set.
	Hence, $b\la$ is entailed by $P$ with respect to $\nccw$. 
	\lipicsEnd
\end{example}

\subsection{Craig Interpolants}

In nonmonotonic formalisms, including Logic Programming, previously drawn conclusions may become invalid in the presence of new information.
For example, if we were to add a rule $c\la$ to the program in Example~\ref{exampleAnswerSets}, then $\{b\}$ would no longer be an answer set.
Instead, the unique answer set would be $\{a,d,e\}$.
This raises problems for (Craig) interpolation for logic programs, as simply replacing the entailment relation $\models$ in Definition \ref{def:general-interpolants} with a nonmonotonic entailment relation $\nc$, which is not transitive in general, may not work. 

Following ideas in the literature \cite{GabbayMaks05}, Gabbay et al.~\cite{GabbayPV11} proposed to adjust the central condition of Craig interpolants to logic programs using two entailment relations. 
We recall this generalized definition of interpolants using two abstract entailment relations.

\begin{definition}\label{def:genInterp}
	Let $\vdash_1$ and $\vdash_2$ be entailment relations and $\phi$, $\psi$ programs such that $\phi\nc\psi$. 
	A \emph{($\vdash_1,\vdash_2$) interpolant for $(\phi,\psi)$} is a program $\chi$ such that $\phi\vdash_1\chi$ and $\chi\vdash_2 \psi$, and $\chi$ uses only non-logical symbols occurring in both $\phi$ and $\psi$.
\end{definition}

This is motivated by the aim to combine a nonmonotonic entailment relation $\nc$ with its \emph{deductive base}, i.e., a logic $L$ with monotonic entailment relation $\vdash_L$ such that (i) $\vdash_L\subseteq \nc$; (ii) for programs $\phi_1$, $\phi_2$, if $\phi_1 \equiv_L \phi_2$, then $\phi_1 \approx \phi_2$ (i.e., $\phi_1$ and $\phi_2$ have the same nonmonotonic entailments); and (iii) if $\phi\nc \chi$ and $\chi\vdash_L \psi$, then $\phi\nc \psi$. 

Two kinds of Craig interpolants are defined. 
The stronger one ($\nc,\vdash_L$) builds on this notion of deductive base, which guarantees among other things that the relation ($\nc,\vdash_L$) is transitive in the sense of (iii). 
The weaker one, ($\nc,\nc$), is implied by the former, but does not guarantee transitivity per se.

Also, two variants of the entailment relation $\nccw$ are considered, one which is aligned with closed world assumption (CWA), i.e., atoms that cannot be proven true are considered false, and one which is aligned with open world assumption (OWA), where no such conclusions on false atoms can be drawn.

In the latter case, it is possible to show for propositional programs that ($\nccw,\htmodels$) interpolants exist by constructing a representation of the minimal models of $\phi$, taking advantage of the fact that HT logic serves as deductive base for answer set semantics and that Craig interpolants as per Definition~\ref{def:general-interpolants} are guaranteed to exist for HT logic \cite{GabbayMaks05}.

However, ASP employs CWA, and in this case, in general, only ($\nccw,\nccw$) interpolants do exist.
Here, the corresponding model representation can be obtained because, in addition to the previously mentioned points, under CWA, the non-logical symbols occurring in $\psi$, but not in $\phi$, can be considered false.
This also ensures, among others, transitivity of ($\nccw,\nccw$) \cite{GabbayPV11}.


\begin{theorem}
	For coherent propositional programs, ($\nccw,\nccw$) interpolants are guaranteed to exist.
\end{theorem}
The stronger notion of ($\nccw,\htmodels$) interpolants only holds for coherent propositional programs under CWA when the non-logical symbols of $\psi$ are a subset of those of $\phi$ \cite{GabbayPV11}.


However, in the first-order case, such interpolants may not exist in general, as the formula representing the answer sets may not be first-order definable.
Yet, first-order definability can be ensured for programs where function symbols are not allowed and where variables are safe. 
Here, \emph{safe} refers to all variables in a rule having to appear in atoms in $\pbody{r}$ in (\ref{l:shortRule}), which ensures that any grounding is finite as well.\footnote{In fact, a slightly more general notion of safe is used to capture a wider range of ASP constructs \cite{GabbayPV11}.}
This restriction is not problematic in practice in the context of ASP with variables, as programs are required to be safe anyway for the grounding, employed in state-of-the-art ASP solvers, such as clingo \cite{GebserKKS19} or DLV2 \cite{dlv2-17}, where the usage of function symbols is also severely restricted to avoid grounding problems.

\begin{theorem}
	For coherent, function-free, and safe first-order programs, ($\nccw,\nccw$) interpolants are guaranteed to exist.
\end{theorem}


\subsection{Uniform Interpolation and Forgetting}

As mentioned in Section~\ref{sec:intro}, uniform interpolation and forgetting are strongly connected in their aim to reduce the language while preserving information over the remaining language.
However, in the context of ASP, mainly forgetting has drawn considerable attention, varying from early purely syntactic approaches to a number of approaches using different means to capture the idea of preserving information over the remaining language \cite{GoncalvesKL23}.
In the following, we briefly discuss the main results and their implications on uniform interpolants.

In accordance with the literature, and to ease presentation, we focus on the propositional case.
Hence, the signature $\sign$ simplifies to propositional variables, i.e., only predicates of arity 0.
We start by refining the definition of uniform interpolants in this context, as before admitting a generic entailment relation $\vdash$.

\begin{definition}\label{def:uniformIntLP}
	Let $V\subseteq\sign$ be a set of atoms, $P$ a propositional program, and $\vdash$ an entailment relation.
	A \emph{uniform $(V,\vdash)$-interpolant for $P$} is a program $P_V$ such that
	\begin{itemize}
		\item $P\vdash P_V$;
		\item $P_V$ uses only atoms from $V$; and
		\item for any program $P'$ using only atoms from $V$, if $P\vdash P'$, then $P_V\vdash P'$.
	\end{itemize}
\end{definition}
Note that for arbitrary nonmonotonic entailment relations, the first condition may in general not hold. 
But, building on results from forgetting, such interpolants can indeed be found.

For forgetting, the focus is on operators that return a program over the remaining language.
Here, unlike uniform interpolation, the atoms to be forgotten are explicitly mentioned.

\begin{definition}\label{def:forgOp}
	Given a class of logic programs $\classP$ over $\sign$, a \emph{forgetting operator (over $\classP$)} is defined as a function $\op:\classP\times 2^{\sign}\to \classP$ where, for each $P\in \classP$ and $V\subseteq \sign$, $\f{P}{V}$, the \emph{result of forgetting about $V$ from $P$}, is a program over $\sign(P)\text{\textbackslash} V$.
\end{definition}
In the literature (see, e.g., \cite{GoncalvesKL23}), operators then aim to preserve the semantic relations between atoms in $\sign(P)\text{\textbackslash} V$ from $P$, such as the dependencies discussed in the following example. 


\begin{example}\label{ex:basic}
	Consider again program $P$ from Example~\ref{exampleAnswerSets}.
	If we want to forget about $d$ from $P$, then the first two rules should remain unchanged, while the latter two should not occur in the result.
	At the same time, implicit relations, such as $e$ depending on $a$ via $d$, should be preserved.
	Hence, we would expect the result $\{a \la \nf b;\  b \la \nf c;\ e \la a\}$.
	Alternatively, consider forgetting about $b$ from $P$. 
	Now, the final two rules remain, whereas the first two would be replaced by one linking $a$ and $c$, i.e., $\{a \la \nf\nf c ;\ e  \la d ;\ d \la a\}$.
	\lipicsEnd
\end{example}
Note that, in the second part, we cannot simply use the rule $a\la c$ instead, absorbing the double negation, as this is not equivalent in general (in terms of HT-models). 
Thus, results of forgetting may be required to be in a class more general than the one of the given program.

Example~\ref{ex:basic} illustrates the main idea of such semantic relations to be preserved, but leaves a formal specification open.
Such characterization is usually semantic, i.e., a characterization of the desired models, which does not fix a specific program, and hence cannot align with a single function.
This is captured by classes of forgetting operators that conjoin several operators based on a common characterization.
Here, since we are interested in the existence of interpolants, we focus on individual operators instead (simplifying some technical material).

A number of different operators and classes of these have been defined in the literature, as well as certain properties with the aim of aiding the semantic characterization.
We recall relevant properties in the context of interpolation and refer for a full account to \cite{GoncalvesKL23}. 

We start with a property (for forgetting operators $\fgt$) comparing the answer sets of a program and its result of forgetting.
For that, given a set of atoms $V$, the \emph{$V$-exclusion} of a set of answer sets $\mathcal{M}$, written $\mathcal{M}_{\parallel V}$, is defined as $\{X\text{\textbackslash} V\mid X\in\mathcal{M}\}$. 

\begin{itemize}[align=left, labelwidth=1.5em, labelsep=0.25em, leftmargin=1.75em]
	\item[\pCP] $\fgt$ satisfies \emph{Consequence Persistence} if, for each program $P$ and $V\subseteq \sign$, we have $\as{\f{P}{V}}=\as{P}_{\parallel V}$.
\end{itemize}
Essentially, \pCP{} requires that answer sets of the forgetting result be answer sets of the original program projected to the remaining language and vice-versa~\cite{WangWZ13}.
There are operators that satisfy \pCP, e.g., those in class $\classF_{\smF}$ \cite{WangWZ13}, whose forgetting results are characterized as the subset of HT-models (projected to the remaining language) such that the condition for \pCP\ holds.
We can show that, if $\fgt$ satisfies $\pCP$, then its forgetting result $\f{P}{V}$ provides a uniform $(V,\nccw)$-interpolant.

\begin{theorem}\label{thm:uniformIntNonMon}
	For propositional programs $P$ and $V\subseteq\sign$, uniform $(V,\nccw)$-interpolants are guaranteed to exist. 
\end{theorem}
The intuition here is that, since \pCP\ holds, we have $P\nccw \fgt(P,V)$ and, likewise, the third condition of Def.~\ref{def:uniformIntLP} holds.
Note that, for the two programs in Example~\ref{ex:basic}, both would constitute uniform $(V,\nccw)$-interpolants 
with $V=\{a,b,c,e\}$ in the first case and $V=\{a,c,d,e\}$ in the second.

Other properties rather focus on HT-models and entailments based on this notion \cite{WangZZZ12}.

\begin{itemize}[align=left, labelwidth=1.5em, labelsep=0.25em, leftmargin=1.75em]
	\item[\pW] $\fgt$ satisfies \emph{Weakening} if, for each program $P$ and $V\subseteq \sign$, we have $P\htmodels\f{P}{V}$. 
	\item[\pPP] $\fgt$ satisfies \emph{Positive Persistence} if, for each program $P$ and $V\subseteq \sign$: if $P\htmodels P'$, with $\sign(P')\subseteq \sign\setm V$, then $\f{P}{V}\htmodels P'$. 
\end{itemize} 

These properties are aligned with Definition~\ref{def:uniformIntLP}, and based on them and classes that satisfy these, e.g., $\classF_{\htF}$, we can also capture interpolants with respect to $\htmodels$ \cite{WangZZZ14}. 

\begin{theorem}
	For propositional programs $P$ and $V\subseteq\sign$, uniform $(V,\htmodels)$-interpolants are guaranteed to exist.
\end{theorem}


Here, again both programs in Example \ref{ex:basic} are uniform $(V,\htmodels)$-interpolants.
For a simple example where these two kinds of interpolants differ, consider $P=\{a \la \nf b;\  b \la \nf a\}$.
Then, $\{a\la \nf\nf a\}$ is uniform $({a},\nccw)$-interpolant, but not a uniform $({a},\htmodels)$-interpolant. On the other hand $\{\}$ is a $({a},\htmodels)$-interpolant, but not a uniform $({a},\nccw)$-interpolant.

Determining whether a given program is a result of forgetting, and hence whether it is a uniform interpolant, is computationally expensive.
It has been shown that, for $\classF_{\htF}$ and $\classF_{\smF}$, this problem is $\PiP{2}$-complete with respect to data complexity (where the size of the input is measured in the number of facts) \cite{WangZZZ14,WangWZ13}.
It has also been argued that, e.g., for $\classF_{\smF}$ the algorithm to compute a result may result in an exponential blow-up in the size of the overall program \cite{WangWZ13}.

Stronger notions of preserving semantic relations have been considered in Forgetting \cite{KnorrA14}.

\begin{itemize}[align=left, labelwidth=1.5em, labelsep=0.25em, leftmargin=1.75em]
	\item[\pSP] $\fgt$ satisfies \emph{Strong Persistence} if, for each program $P$, $V\subseteq \sign$, and program $R$ with $\sign(R)\subseteq \sign\setm V$, we have $\as{\f{P}{V}\cup R}=\as{P\cup R}_{\parallel V}$. 
\end{itemize}
The idea is to preserve the answer sets no matter what other program is added to the given program and its result of forgeting.
This turns out to be not possible in general \cite{GoncalvesKL-ECAI16}, and concise conditions when such forgetting is possible and classes of operators approximating this have been investigated \cite{GoncalvesKLW20}.
Complementarily, when the added programs $R$ are restricted to sets of facts, forgetting is always possible with applications in modular ASP \cite{GoncalvesJKLW19}.
In both cases, it has been shown that determining whether a given program is a result of forgetting for these classes is more demanding, i.e., $\PiP{3}$-complete (data complexity). 
This confirms that such notions of forgetting indeed go beyond uniform interpolation.
Also, syntactic operators have been considered for these classes, but their construction may again, in the worst case, result in an exponential blowup in the size of the program \cite{BertholdGKL19,Berthold22,GoncalvesJKL21}.
Corresponding algorithms of these, as well as those building on constructing a representative of the semantic characterization have been implemented in the Web tool ForgettingWeb \cite{forgettingweb}.

One open problem for future research is how to tackle in general finding uniform interpolants/forgetting for first-order programs with variables in a wider sense (without requiring the grounding up-front).
Possible avenues include extensions of the language of programs, in the line of the work by Aguado et al., which allows Strong Persistence to always hold \cite{AguadoCFPPV19}.
Alternatively, connections to related notions on abstraction in ASP might be pursued, that rather aim at abstracting different elements of the language into a conjoined abstract one, either on the object level or in between different concepts \cite{SaribaturE21,SaribaturKGL24,SaribaturW23,SaribaturES21}.
A further option arises from recent work where first-order Craig-interpolation is used to synthesize strongly equivalent programs given some background knowledge (in the form of a program) with an implementation using first-order reasoners \cite{HeuerW24}.

\section{Conclusion}\label{sec:conclusion}

We discussed interpolation in two relevant KR formalisms, DLs and logic programming. In the context of DLs, we focussed on uniform interpolation at the level of ontologies, and on Craig interpolation at the level of concepts. Uniform interpolation is strongly related to the topic of \emph{forgetting}, which has been extensively studied in many settings.     
Interpolation has also been extensively studied for modal logics and classical logics, which are formalisms relevant for KR that are discussed in other chapters of this volume.
Interestingly, there is comparably little research on interpolation in other formalisms relevant for KR.

With logic programs, we discussed one formalism with a nonmonotonic entailment relation. Nonmonotonic reasoning plays indeed a central role in many KR applications, and  many more nonmonotonic KR formalisms have been proposed. Examples include \emph{default logic}~\cite{DEFAULT_LOGIC}, \emph{belief revision}~\cite{BELIEF_REVISION}, \emph{argumentation frameworks}~\cite{ARGUMENTATION_FRAMEWORKS}, as well as semantics based on \emph{circumscription}~\cite{CIRCUMSCRIPTION} and on \emph{repairs}~\cite{REPAIR_BASED_SEMANTICS}. Interpolation and the related notion of forgetting have been researched in some of these formalisms as well, but many problems remain open.
The investigations in \cite[Chapter~9]{DBLP:books/sp/GabbayS16}, \cite[Chapter~6]{DBLP:books/sp/Schlechta18a}
and~\cite{DBLP:conf/jelia/Amir02} discuss interpolation in nonmonotonic logics on a more abstract level. Interpolation and Beth definability for default logics is discussed in~\cite{DBLP:conf/jelia/CassanoFAC19}, and forgetting for defeasible logic has been investigated in~\cite{DBLP:conf/lpar/AntoniouEW12}.

While there are implemented systems for computing uniform interpolants of DL ontologies, Craig interpolation and Beth definability have so far been investigated mostly on a theoretical level. Given the increased interest in explainability~\cite{EVONNE,DBLP:conf/kr/AlrabbaaBFHKKKK24} and supervised learning~\cite{SURVEY_LEARNING_DLs} for DLs in the recent years, which are relevant use cases of Craig interpolation, this situation might change in the future. Indeed, Craig interpolation at the level of ontologies, implemented via a reduction to uniform interpolation, is used in the ontology explanation tool \textsc{Evee} to create so-called \emph{elimination proofs}~\cite{DBLP:conf/kr/AlrabbaaBFHKKKK24}.

\section*{Acknowledgments}
The authors would like to thank Michael Benedikt and Balder ten Cate for valuable comments on earlier versions of this chapter. M.\ Knorr acknowledges partial support by project BIO-REVISE (2023.13327.PEX) and by UID/04516/NOVA Laboratory for Computer Science and Informatics (NOVA LINCS) with the financial support of FCT.IP.
\addcontentsline{toc}{section}{Acknowledgments}

\bibliography{taci,main-ci-kr}

\end{document}